\documentclass[a4paper,12pt]{report}

\usepackage{bold-extra}
\usepackage{float}
\usepackage{graphicx,fancyhdr,amsmath,amssymb}
\usepackage[USenglish]{babel}
\usepackage{listings}
\usepackage{longtable}
\usepackage[usenames,dvipsnames]{xcolor}
\usepackage{pdfpages}
\usepackage{listings}
\usepackage{verbatim}
\usepackage[nottoc]{tocbibind}
\usepackage{caption}
\usepackage{subcaption}

\usepackage{amsthm}
\newtheoremstyle{parenbold} % Name
  {
  1em
  }           % Space above
  {}                    % Space below
  {
  \itshape
  }  % Body font
  {}                    % Indent amount
  {
  \bfseries
  }           % Theorem head font
  {}                    % Punctuation after theorem head
  {1ex}                 % Space after theorem head
  {\thmname{#1 }\thmnumber{#2}\thmnote{ [#3]}} % Theorem head spec
\theoremstyle{definition}
\newtheorem{definition}{Definition}[section]
\theoremstyle{parenbold}
\newtheorem{proposition}{Proposition}[section]
\newtheorem{hypothesis}{Null Hypothesis}[section]

\usepackage[% % Customisation of hyperref; added by Luke
    colorlinks,
    linkcolor=blue,
    citecolor=brown,
    backref=page,
    urlcolor=blue,
    linktoc=all,
    pdfstartview=FitH
]{hyperref} % For hyperlinks and URLs
\hypersetup{final}

\usepackage{url}  % added by Luke
\usepackage[normalem]{ulem} % \sout{...} - strike out; added by Luke

\usepackage[square,sort]{natbib} % added by Luke
\newcounter{expcount}

\renewcommand{\to}{\rightarrow}

\newcommand{\pc}[1]{\texttt{#1}}

\renewcommand{\phi}{\varphi}

\renewcommand{\implies}{\Rightarrow}

\newcommand{\metaai}{Metagol\textsubscript{\emph{AI}}}
\newcommand{\metast}{Metagol\textsubscript{\emph{PT}}}
\newcommand{\metart}{Metagol\textsubscript{\emph{RT}}}
\newcommand{\metagol}{Metagol}
\newcommand{\predsig}{\mathcal{P}}
\newcommand{\consts}{\mathcal{C}}
\newcommand{\fovars}{\mathcal{V}_1}
\newcommand{\hovars}{\mathcal{V}_2}
\newcommand{\kb}{\mathcal{B}}
\newcommand{\defc}{\mathcal{D}}
\newcommand{\meta}{\mathcal{M}}
\newcommand{\exs}{\mathcal{E}}
\newcommand{\exspos}{\mathcal{E}^+}
\newcommand{\exsneg}{\mathcal{E}^-}
\newcommand{\entails}{\vDash}

\newcommand{\<}{\langle}
\renewcommand{\>}{\rangle}
\newcommand{\ot}{\leftarrow}
\newcommand{\Htwo}{H_2^2}

\title{
	Refinement Type Directed Search for Meta-Interpretive-Learning of Higher-Order Logic Programs
}

\author{Rolf Morel}
\date{}

\begin{document}

%\includepdf[pages=-]{SubmissionTemplate_withForm2D.pdf}
%\includepdf[pages=-]{SubmissionTemplateFilledIn.pdf}

\begin{titlepage}
  \centering

  \includegraphics[width=0.23\textwidth]{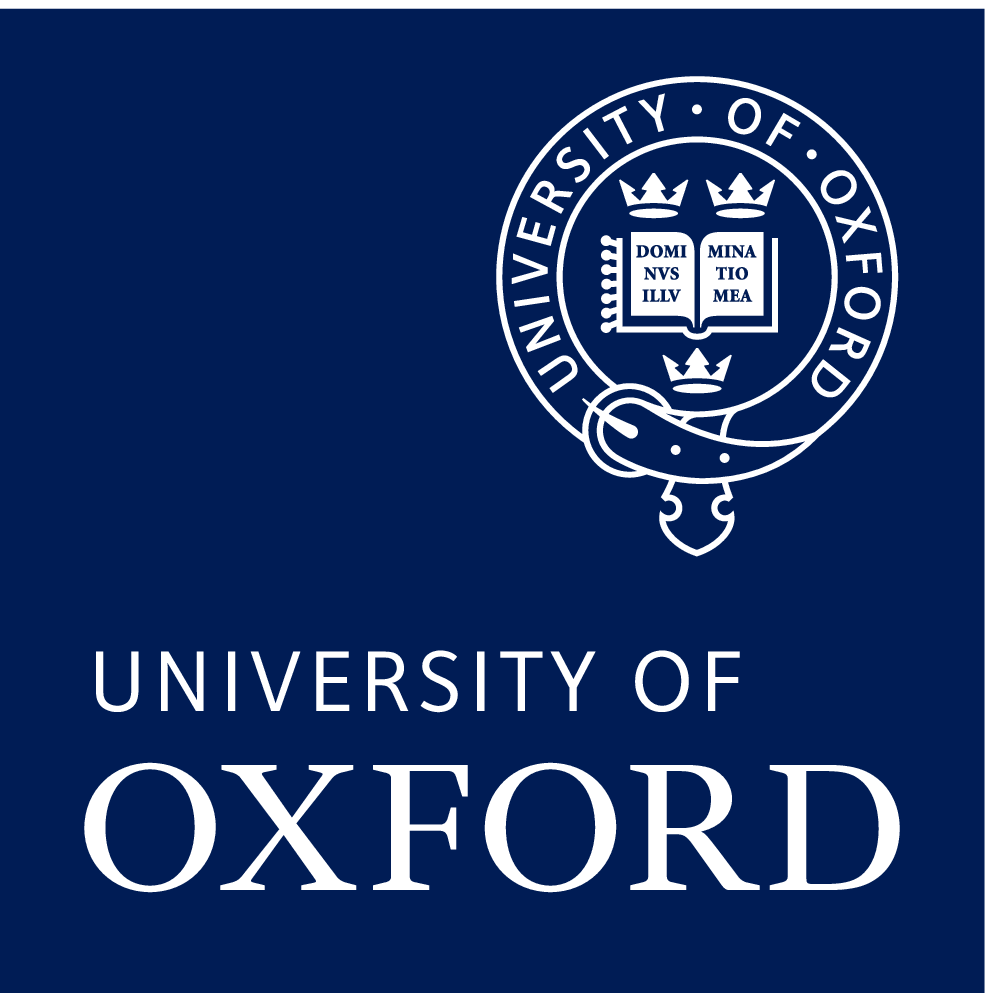}\par\vspace{1cm}
  \vspace{0.5cm}
  {\scshape\LARGE University of Oxford \par}
  \vspace{0.5cm}
  {\scshape\Large MSc in Computer Science 2017-18 \par}
  %\vspace{0.5cm}
  %{\scshape\Large Project Dissertation \par}
  \vspace{0.5cm}
  {\LARGE\bfseries 
Refinement Type Directed Search for \\Meta-Interpretive Learning of \\Higher-Order Logic Programs
  \par}
  \vspace{3cm}
  
  {\Large\itshape Rolf Morel \par}
  \vspace{0.25cm}
  {\Large\itshape St. John's College \par}
  \vspace{0.25cm}
  {\Large\itshape supervised by Prof.~Luke Ong\\and Dr.~Andrew Cropper \par}
  \vspace{2cm}
  
  {\large A thesis submitted for the degree of \par}
  \vspace{0.25cm}
  {\large\itshape Master of Science Computer Science \par}
  %\vspace{0.25cm}
  %{\large Trinity 2018\par}

  % Bottom of the page
\end{titlepage}

%\maketitle
\thispagestyle{empty}

\chapter*{Abstract}

The program synthesis problem within the Inductive Logic Programming (ILP)
community has typically been seen as untyped.
We consider the benefits of user provided types on background knowledge.
Building on the Meta-Interpretive Learning (MIL) framework, we show that type
checking is able to prune large parts of the hypothesis space of programs.
The introduction of polymorphic type checking to the MIL approach to logic 
program synthesis is validated by strong theoretical and experimental results,
showing a cubic reduction in the size of the search space and synthesis time,
in terms of the number of typed background predicates.
Additionally we are able to infer polymorphic types of 
synthesized clauses and of entire programs.
The other advancement is in developing an approach to leveraging
refinement types in ILP. Here we show that further pruning of the search
space can be achieved, though the SMT solving used for refinement type checking
comes at a significant cost timewise.

%\chapter*{Acknowledgements}
%\todo{Maybe??}

\tableofcontents

\chapter{Introduction}

In the last decade there has been a great surge in the effective application of 
Artificial Intelligence techniques to a great number of 
practical problems, e.g.~image classification \citep{deng2009imagenet},
code completion \citep{Raychev:2014:CCS:2666356.2594321},
as well as autonomous navigation \citep{bojarski2016end}.
Though these advances in AI technology have made previous instances of the above 
problems tractable, many of the techniques leveraged are not able to give a useful 
explanation of why decisions are being made (a prime example being artificial 
neural networks), and struggle to generalize over their entire input space
\citep{gulwani2015inductive}.
For some problems, like in the case of autonomous driving cars, the ability to
understand why certain decisions will be/were made could be of utmost importance 
with regard to safety and liability issues.

\paragraph{Comprehensibility}
The problem with the alluded to methods is that the generated programs 
do not necessarily have a description that allows humans to understand and
effectively reason about the programs (e.g.~the edge weights in neural networks 
do not give a ready interpretation of program behaviour). The focus
of these approaches is on improving according to \cite{michie}'s \emph{weak}
criterion for machine learning, i.e.~improving predictive accuracy.
In contrast, programs described by programming languages do have
semantics that are comprehensible to both machines and humans, 
which corresponds to Michie's criterion for \emph{strong} machine learning.
From this it follows that a fruitful approach to explainable AI might be to
generate computer programs expressed in programming languages, 
i.e.~\emph{code programs}. 
Programming languages also strongly emphasize
the composability of programs, motivating the usage of the term 
\emph{synthesis} for the generation of code programs.
Logic and functional languages in particular are good candidates for synthesis 
due to their high level of abstraction being associated with smaller programs. 
Recent work on \emph{ultra-strong} machine learning \citep{Muggleton2018}
has shown that learning logic relational programs can extend humans' capacity 
for understanding training data to a level beyond unassisted study.

\paragraph{Inductive synthesis}
Code program synthesis is not only useful in terms of explainable AI, it is 
also of interest for the sake of automating programming.
Code synthesis holds promise as a tool for developers to become
more efficient, but has also proven effective in empowering non-programmers in 
obtaining programs that fit their needs, e.g.~learning spreadsheet transformations \citep{Singh2012a}. For these
users of program synthesis the synthesis problem is posed as finding a program
that accounts for the examples that the user gives (in terms of inputs and
outputs) and which is mostly likely to appropriately generalize from the
examples \citep{Gulwani2010}.
The paradigm that combines both automatic generation of code programs and
learning from examples is known as \emph{Inductive Programming} \citep{kitzelmann}.
Versus general machine learning, 
inductive programming typically relies on very few examples to learn from,
e.g.~the single example $f([[1,2,3],[4,5]]) = [[1,2],[4]]$ is enough to induce
a program that drops the last element from input lists.
For programmers this paradigm is useful for problems ranging from 
which constants to fill in for edge cases (see \citep{Solar-Lezama2009}), up to
the potential to discover novel efficient algorithms \citep{Cropper2018}.

\paragraph{Logic programs}
The above considerations motivate the study of Inductive Logic Programming (ILP),
"a form of machine learning that uses logic programming to represent examples, background knowledge, and learned programs" \citep{Cropper2017}, 
an active field of study for almost thirty years \citep{Muggleton1991}.
The specification of an inductive logic problem is in terms of 
the examples the program needs to satisfy, along with background information 
in the form of asserted facts and available program fragments,
and possibly rules about the structure of the program to be synthesized.
A particularly powerful approach developed recently is that of the 
\emph{Meta-Interpretive Learning} (MIL) \citep{Muggleton2014}.
One of the major benefits of this approach (versus virtually all other existing 
synthesis systems) is that it is able do \emph{predicate invention}, that is,
it is able to construct helper predicates/functions.
%An extension of the MIL framework, by \citep{Cropper2016}, introduced the
%use of higher-order predicates (and inventions thereof) to the framework,
%which allows for more succinct programs.

The MIL framework essentially takes the user's problem
description and sees it as a specification of a search space of possible 
programs to consider.
The time required for synthesis is most influenced by how large the search space
is that is traversed.
The implementation of the MIL framework is written in Prolog,
a logic programming language that does not consider the types 
of its terms. The MIL algorithm in turn does not make use of the 
types of predicates during its search.

\paragraph{Contribution}

The thesis of this document is that types are a significant feature for the
MIL approach to code program synthesis.  
Types provide an effective way of pruning out nonsensical programs in the search
space, i.e.~programs already rejected by type checking.
Different gradations of types
are considered, namely polymorphic types, i.e.~types with type parameteres,
and refinement types (with polymorphism), i.e.~types with a proposition restricting
inhabitants.
The second system has strictly more accurate types than the first, 
but the type checking comes at a considerable cost.
For polymorphic types we show significant benefits both theoretically and 
experimentally, 
showing a cubic reduction in the size of the search space and synthesis time,
in terms of the number of typed background predicates.
Additionally we are able to infer polymorphic types of 
synthesized clauses and of entire programs.
For refinement typing the foremost contribution is in the 
introduction of refinement types to ILP, while the experimental results 
indicate more work is needed for refinement type checking to be truly effective.

\paragraph{Document organization}
This dissertation consists of seven chapters. The first chapter presents a brief
introduction to the project. The other chapters are described as follows:
chapter 2 contains a literature review of code synthesis approaches
with a focus on types and ILP. Chapter 3 discusses
in detail the existing MIL systems.
Chapters 4 through 7 represent the novel contributions of this project.
Chapter 4 discusses where MIL can benefit from the introduction of types.
The subsequent chapter discusses how
polymorpic type checking is introduced to the existing system. 
Chapter 6 deals with refinement types and with how Satisfiability Modulo
Theories (SMT) solving is leveraged for type checking. 
Both chapter 5 and 6 have sections presenting theoretical and experimental results.
The conclusion, chapter 7, considers the implications of the 
experimental results and relates this outcome to future work.

\chapter{Review of Program Synthesis}

\iftrue
Following on from the arguments made in the introduction regarding the 
significance of program synthesis,
this chapter opens with discussing the strong relation
between program synthesis and the notion of proof search.

Subsequently the presentation switches to a literature review.
Different aspects of the synthesis problem are highlighted, such as the
differences in problem definitions, various application areas and
multiple paradigms.
We do not aim for comprehensiveness, instead we review a selection of disparate
approaches in the literature based on contrasting features.

The focus is on code program synthesis approaches for synthesis of functional 
programs, usually based around typing information,
and synthesis of logic programs, where we mainly look at the
Inductive Logic Programming (ILP) approach. Methods outside of
these areas are also touched upon.

%\rm{Andrew's thesis: "learning efficient robot strategies In Chapter 7, we use
%Metaopt to learn efficient resource complexity robot strategies [24]. In contrast to traditional AI planning [114], which involves the generation of a plan as a sequence of actions transforming a particular initial state to a particular final
%state, a strategy can be viewed as a potentially infinite set of plans, applicable to a class of initial/final state pairs [24]. Specifically, we make these contributions:
%• We introduce the resource complexity minimisation problem, a variant of the cost minimisation problem, where resource complexity is a user- defined measure of the efficiency of a robot strategy.
%• We conduct experiments on learning robot librarian, postman, and sorter strategies which show that Metaopt learns minimal resource complexity strategies in all cases.
%This work is the first to demonstrate learning efficient robot strategies [24]."}
%\rm{The Related work chapter in Andrew's thesis could also provide useful information/references}

\section{Program Synthesis as Proof Search}
%\section{The Proof Theoretical View}

There is a strong correspondence between program synthesis and automated
theorem proving. In program synthesis there is a specification for
which we try to find a program satisfying the specification.
The user might provide helper functions to be used in the synthesized program.
In automated theorem proving we state a proposition and want to find a proof
for the validity of the formula. The user might provide lemmas that may
be used in the proof.

As we will see later, many specifications can be expressed as a type that
the synthesized program needs to satisfy. 
The connection between program synthesis and proof search is most 
eloquently expressed by the Curry-Howard correspondence \citep{sorensen2006lectures}.
This correspondence states that there is a direct mapping between 
programs (expressed in type theory) and mathematical proofs.
For every type theory, a formal system for describing (functional) programs 
with their types, there is a mapping between the types of the programs and logical formulae.
From this correspondence it follows that type checking is the same as
checking that a proof follows the rules of formal logical system 
(e.g.~natural deduction).

This connection means that when we are interested in program synthesis we 
should be aware of the work that exists ``on the other side'' of the 
correspondence, namely in (semi-)automated theorem proving.
On the side of theorem proving there is an interesting distinction made
between systems that fully automate the construction of proofs 
\citep{bibel2013automated} and 
of so called proof assistants (e.g.~\cite*{the_coq_development_team_2018_1219885},
and Isabelle \citep{Nipkow:2002:IPA:1791547}), where software assists users in
writing down proofs in a formal system. One form of assistance provided
is the ability to do proof search. Whilst the ideal of theorem proving is full
automation, for many non-trivial logical systems this is simply not tractable 
(in general, it is even undecidable to determine the truth of all propositions).
Proof assistants are instead able to do proof search based on a (partial) proof,
and lemmas already constructed by the user.

The idea of having the user
guide proof automation is recognizable program synthesis as
providing background knowledge, in the form of helper functions.
How a user of a proof assistant asks the system to find a proof, and upon
failure will try to provide additional guidance by providing additional
lemmas (or working further on the proof) is also in clear relation to how
users work with synthesis systems. 

For this project we work in the setting of logic programs. Logic programming
is at an interesting intersection of programming and theorem proving, as 
execution is performed based on SLD-resolution, the algorithm behind
first-order theorem provers. 
For this document we will hence often rely on the terminology of 
proofs and proof search.

\section{Dimensions of Program Synthesis}

In this section we review several important features of program synthesis.
Based on a survey of the literature (focused on type-based synthesis and approaches
to ILP), we look at the domains and applications areas where synthesis
has been successfully applied.
Next we look at the different ways that the synthesis problem can be specified
by the user.
Following on we look at the types of programs that can be learned by different
approaches. We finish on a discussion of the recognisable paradigms within
the literature.
We follow \citep{Gulwani2010} in identifying the
key aspects of an approach to the synthesis problem as
capturing the user's intent in a specification, the space of
programs that is considered, and the approach to searching this space.

\subsection{Synthesis Domains/Application Areas}

In this section we give an idea of the diversity of problems addressable by
program synthesis, based on a selection of successful approaches in the literature.
The review paper by \citep{Gulwani2010} is great resource for entering upon
the field of program synthesis, and unless otherwise annotated the reference
for the below discussion is this paper.

Bitvector algorithms ``typically describe some plausible yet unusual operation
on integers or bit strings that could easily be programmed using either a 
longish fixed sequence of machine instructions or a loop, but the same thing 
can be done much more cleverly using just four or three or two carefully chosen 
instructions whose interactions are not at all obvious until explained or fathomed''
\citep{Warren:2002:HD:515297}. In synthesis of bitvector algorithm initial
approaches used straight-forward brute-force
search \citep{Massalin:1987:SLS:36177.36194}, while newer approaches
are to leverage SMT solver reasoning \citep{Jha:2010:OCP:1806799.1806833}.

Other systems are able to fill in holes left in programs. The Sketch system
\citep{Solar-Lezama2009} deals with finding the appropriate values for
single value holes in imperative programs.
The holes are typically boundary conditions, e.g.~for loops.
In template-based synthesis \citep{Srivastava2013} the conditions on holes
are less-restricted in that holes can be filled by arbitrary expressions.
Templates are a way for the user to provide their insight
to the synthesis system, by writing code or invariants with holes.
Constraints are generated for these holes and SMT solvers are used to 
find solutions.

Sometimes we are able to give a very precise specification for algorithms.
For example, we may be able to express the condition for properly implementing
critical sections and shared variables for mutual exclusion algorithms. 
A synthesis system may then be able synthesis new concurrent algorithms by
properly inserting lock acquiring and releasing statements. \citep{bar2003automatic}
More generally we might be able to very precisely describe how function behave.
For functional programs a good place to assert such a specification is in
the type of the program \citep{frankle2016example}.
In \citep{polikarpova2016program} non-trivial algorithms over binary-search trees
are synthesized from specifications expressed as refinement types. The synthesis
technique derives from type checking rules for the programming language that
is considered.

The are a couple of domains that are only well suited for synthesis from
examples. We highlight string and matrix transformations.
A string transformation is a mapping on strings.
For example ``tony.hoare@cs.ox.ac.uk'' might be
mapped to ``Hoare, T., UK''. One of the applications of string transformation 
is to learn spreadsheet operations \citep{Singh2012}. 
String transformations are also an often considered problem in ILP,
e.g.~\citep{cropper2016data}.
Similarly one can try to learn matrix transformations, by giving input-output
examples. These examples are typically larger than string transformations,
making it a good benchmark test for synthesis systems \citep{Wang2017}.

Another avenue explored is inventing strategies for robots performing tasks.
In \citep{Cropper2016} a very high-level strategy is learned for 
serving either tea or coffee for a table of cups. The system is provided with
examples of good behaviour which are used to generate a higher-order logic 
program describing the actions that the robot should perform.

%\begin{center}
%\begin{tabular}{ p{3cm}|p{2cm}|p{2cm}|p{2cm}|p{3cm} }
%Domain&
%\textbf{example based} & \textbf{spec.~based (logic/type)}& \textbf{sketch based} &\textbf{other}
%\\\hline
%Mutual exclusion algorithms
%&-&\citep{Gulwani2010}&\citep{Gulwani2010}&-
%\\\hline
%String transformations
%&\citep{Singh2012}\citep{Singh2012a}...&-&-&
%Very often used as motivation. Includes number transforms.
%\\\hline
%Matrix transformations
%&\citep{Wang2017}\citep{Polozov2015}&-&-&
%\\\hline
%Classifiers 
%&\citep{Kaminski2018}&-&-&
%e.g. east west bound trains
%\\\hline
%\\\hline
%Data structure manipulation
%&\citep{frankle2016example}\citep{Feser2015}\citep{osera2015type}&\citep{polikarpova2016program}&-&
%\\\hline
%Finding program inverses
%&-&-&-&\citep{Gulwani2010}
%\\\hline
%Automated Debugging
%&-&-&~&\citep{Gulwani2010}
%\\\hline
%\end{tabular}
%\end{center}

\subsection{Specifications}

From the user's perspective one of the most important features of a synthesis
system is which problems it is able to solve. In this subsection we present
different ways of specifying the synthesis problem. We distinguish two
main aspects to specification, namely how to specify the \emph{goal}
and how to provide \emph{background knowledge}.

\paragraph{Goals}
The \emph{goal} of a synthesis problem is a program that satisfies the 
specification.
The are multiple ways to describe what is expected of the goal.

There is the programming-by-example specification, where the conditions
on the goal are stated as input-output examples \citep{metagol}. Examples
are split in positive examples, ones that the synthesized program needs
to satisfy, and negative examples (also called counter-examples)
which should not be entailed by the program.
A related specification of the goal is programming-by-demonstration \citep{Lau2000}:
in addition to providing input-output examples a (partial) trace is provided
of the transformations that turned the input into the output.

As discussed in the previous section, there is the Sketch/templates
approach to specification by writing programs with holes.
In essence the program goal is already partially given by the user and only
small fragments need to be filled in.

A very general approach is to say that any (set of) logical proposition(s)
might be used as a specification. The previous section's mutual exclusion
conditions would fall in this category. A more rigid framework follows
from only allowing specifications over a function's arguments.
By the Curry-Howard correspondence these propositions may also be stated
as function typing \citep{frankle2016example}. A recent development is that even 
examples may be encoded in types \citep{osera2015type}, yielding types as a very
powerful specification vehicle.

\paragraph{Background knowledge}

There are systems that approach synthesis with only a specification. Such 
systems are forced to perform searches over very larges program spaces,
the exhaustive search of small lambda-terms satisfying a type, by
\citep{Katayama2005}, being such an example. 
More common is to accept guidance from the user, in which case we call the
provided hints the \emph{background knowledge}. 

In inductive logic programming the user may provide background knowledge
in the form of facts, defined Horn clauses (which may be higher-order)
\citep{Raedt2010}.
In meta-interpretive learning \citep{muggleton2014meta} in addition metarules
are supplied. Metarules define structure for clauses which are \emph{invented}
by the synthesis system. In chapter 3 we will at these metarules in more
detail.

In the setting where we are trying to synthesis functional programs from
specification embedded in types background knowledge takes on the following
forms.
Type declarations informing the system of the data structures over which to
operate. Type declarations for helper functions (with or without the actual
definition of the functions). For encoding more precise properties refinement
types are leveraged. 

As a final example of background knowledge we highlight that the Sketch and
template approaches to synthesis blur the line between providing goal specifications
and background knowledge. Partial programs are more typically background
knowledge, but are here used as the primary means of specification.

\subsection{Paradigms and Types of Programs Learned}

In this subsection we look some of the paradigms for program synthesis. For
the separate approaches we highlight the types of programs that have been 
synthesized by the method.
The most straightforward approach to synthesis to enumerative the entire space
of programs that your method is willing to consider, e.g.~\citep{Katayama2005}
searches over all lambda-terms. We look at several more sophisticated methods.
We will look at Inductive Logic Programming here as it will feature prominently
in the remainder of the document.

\paragraph{Maintaining consistent programs}

One idea is to maintain the space of programs that are consistent with examples.
With no examples the entire space of programs works. By interatively adding
examples one can start reducing this space of consistent programs.
Version Space Algebra \citep{Lau2000} is a powerful approach whereby the
space of programs (version space) is maintained by a partial ordering of
programs (usually by generality), which can be fully represented by the maxima and
minima of the ordering. An update function is able to shrink these boundaries
for each additional example considered.

Other approaches that are able to use a succinct representation of the space
of solutions is the work by \cite{Gvero2013} where types are used to 
represent classes of expressions that are candidates for code completion
queries in IDEs. In \citep{Singh2012}, an exponentially sized space of
consistent string transformation is maintained in polynomial space by a
clever sharing of shared sub-expressions.

%Synthesizing Number Transformations from Input-Output Examples
%\citep{Singh2012a}

\paragraph{Constraint solving}

A very general approach to synthesis is to convert the problem to (logical) 
constraints. These constraints are written in a language of an off-the-shelf
solver (in particular SMT solvers). The results of the solver
are then used to construct program solutions.

In the Sketch/template approach to synthesis the holes in programs
are surrounded by code that imposes conditions on the possibilities for such holes
\citep{Solar-Lezama2009} \citep{Srivastava2013}.
In Constraint-Based Synthesis of Datalog Programs \citep{B2017} the derivation
algorithm for (first-order) logic programming is encoded as logical constraints,
with additionally that the predicates symbols are allowed to vary, again
subject to constraints. To make problem tractable the constraints only
encode derivations of a bounded length. Any solution found satisfying the 
constraints for partial derivations is then checked for being an actual solution
by using the found program to build a derivation in Datalog itself.
If the program does not work an additional constraint is generated excluding
the program from being considered again.

\paragraph{Inductive logic programming}

We have that in ILP the norm is to learn untyped logic programs.
Logic programs consist of Horn clauses, implications with a single atom in
the consequent \citep{Raedt2010}. 
These programs are commonly interpreted as either Prolog, or
Datalog code. Prolog represents the SLD-resolution approach to program execution,
and comes with features that make the language Turing-complete, whilst Datalog
uses grounding of first-order logic to propositional logic to determine
entailment (which is decidable). While ILP systems primarily are used to
synthesize first-order programs, with an important feature being recursive programs,
recent work makes invention of higher-order programs possible
\citep{Muggleton2015}. 
This document will further discuss the Meta-Interpretive Learning (MIL)
framework as a particularly strong approach to ILP.
A simplistic approach to typing in MIL was already considered in 
\citep{Farquhar2015} in order to learn proof strategies. The system only
support simple non-polymorphic types and hardcodes types in metarules.

%Learning higher-order logic programs through abstraction and invention
%\citep{Cropper2016}
%Meta-interpretive learning of higher-order dyadic datalog: predicate invention revisited
%\citep{Muggleton2015}
%Logical minimisation of meta-rules within meta-interpretive learning
%\citep{Cropper2015}
%%Exploiting Answer Set Programming with External Sources for Meta-Interpretive Learning
%%\citep{Kaminski2018}

\paragraph{Typing}

As already stated, the specification of functions is well addressed by
types. There are a number of approaches in the literature on using types
to synthesize functional programs.
The norm is to synthesis programs according to an (existing) type system,
with the guarantee that the programs type checks.
Important features include being parametric polymorphism, algebraic data types,
and structural recursion over these types.
Modern systems are able to utilize refinement types as specification, but also
direct the search \citep{polikarpova2016program}. 

Type directed synthesis usually takes an approach that is very similar to
theorem provers. As the program will need to conform to a type according to
the rules of a type system, these rules are actually used to direct the search
\citep{osera2015type}.
The type specification is decomposed according the rules and non-deterministic
choices are made when the premises of the rules need information that is
not captured in the conclusion, i.e.~the current program fragment to be proven.

It is interesting to note that type directed synthesis is quite flexible.
In \citep{frankle2016example} it is noted that examples can encoded in types, meaning
that types are expressive enough to capture logical requirements on programs
as well as inductive synthesis from examples.
In \citep{Gvero2013} types are used direct the generation of code completions
in IDEs. Here types are viewed as set of consistent expressions all being
candidates to be enumerated.

%\subsubsection{Non-inductive/generalization}

%Learning explanatory rules from noisy data \citep{Evans2018}

\paragraph{Arbitrary DSLs}

A more ambitious approach is to synthesis programs not restricted to a particular
programming language.
In work by \citep{Wang2017} a domain expert provides a domain specific language
(DSL) and the end-user provides examples. Along with the DSL concrete
semantics (an evaluation function) and abstract semantics (a mapping to abstract
values, e.g.~value ranges) are provided. The domain experts chooses the 
appropriate abstractions made available to the system.

The approach uses finite tree automata (FTA) to encode abstract syntax trees (ASTs)
of the DSL with predicates on the nodes encoding their abstract semantics.
The FTA is used to maintain a set of programs
whose abstract semantics are consistent with the examples.
A ``best'' program is selected among the accepted programs and is checked
against the examples. If unsuccessful the tree automata is modified such that
more nodes with abstract semantics become available and with the guarantee
that the previously selected programs are no longer successful.
This is iterated until a successful program is found.
It is shown in the paper that this system is able to outperform multiple
other purpose built systems. However, a severe limitation of this approach
is that it does not handle DSLs with binders, e.g.~lambda terms and let bindings.

\fi

\chapter{Untyped Meta-Interpretive Learning}

This chapter contains an overview of the existing meta-interpretive learning (MIL) 
approach to one variant of the inductive logic programming problem.
First, we state prerequisites before presenting a formal definition 
of the problem addressed by MIL.
Subsequently the MIL framework is explained, along with the central role of
metarules. Finally, the high-level algorithm is presented in the form of
\metaai{}, an implementation of MIL which incorporates
a powerful extension to allow for higher-order abstractions.

\section{Logic Prerequisites}

We work within the framework of logic programming. A reader unfamiliar with
this topic is referred to \citep{nienhuys1997foundations} for a comprehensive treatment.
The primary logic programming features we will assume familarity with are
logical variables, unification on logical variables, and SLD-resolution (along
with seeing a SLD-tree as a derivation/proof of a goal atom).
We use the language of logic to introduce the requisite concepts. 
The definitions in this chapter mainly follow those of \cite{Cropper2017}'s PhD thesis.

\paragraph{First-order}
A \emph{variable} is a character string whose initial letter is uppercase.
\emph{Function and predicate symbols} are character strings whose initial letter is lowercase.
The \emph{arity} $n$ of a predicate/function symbol $p$ is the number of arguments  
it takes and is denoted as $p/n$.
The \emph{predicate signature} $\predsig{}$ is the set of predicate symbols with
arity greater than 0.
A \emph{constant} is a function symbol with arity zero. 
The \emph{constant signature} $\consts{}$ is the set of constant symbols.
A variable which can be substituted by a constant or 
function symbol is called a \emph{first-order}.
The set of first-order variables is denoted as $\fovars{}$.
A \emph{term} is a variable, a constant symbol, or a function symbol of arity $n$
immediately followed by a bracketed $n$-tuple of terms.
A term which contains no variables is called \emph{ground}.
A formula $p(t_1, \dots , t_n)$, where $p$ is a predicate symbol
of arity $n$ and each $t_i$ is a term, is called an \emph{atom}.
An atom is ground if all of its terms are ground.
The $\neg$ symbol is used for negation.
A \emph{literal} is an atom $A$ or its negation $\neg A$.

\paragraph{Clauses}
A \emph{clause} is a finite disjunction of literals.
Each variable in a clause is implicitly universally quantified. 
A clause that contains no variables is ground.
%Simultaneously replacing variables $v_1, \ldots, v_n$ in a formula with
%terms $t_1, \ldots , t_n$ is called a substitution and is denoted as 
%$\theta = \{v_1/t_1, \ldots , v_n/t_n\}$.
%A substitution $\theta$ unifies atoms $A$ and $B$ in the case $A\theta = B\theta$.
%A clause $C$ $\theta$-subsumes a clause $D$ whenever there exists a substitution θ such that Cθ ⊆ D.
Clauses with at most one positive literal are called \emph{Horn clauses}.
A Horn clause with exactly one positive literal is called a \emph{definite clause}:

\begin{definition}
A (first-order) \emph{definite clause} is of the form: 
\[ A_0 \leftarrow A_1, \ldots, A_m \]
where $m \geq 0$ and each $A_i$ is an atom of the form
$p(t_1, \ldots, t_n)$, such that $p \in \predsig{}$ and
$t_i \in \consts{} \cup \fovars$. 
The atom $A_0$ is the \emph{head} and the conjunction $A_1, \ldots, A_m$ is the \emph{body}.
\end{definition}

A definite clause with no body literals is called a \emph{fact}. 
A Horn clause with no head, i.e. no positive literal, is called a \emph{goal}.

\paragraph{Higher-order}
For higher-order logic, the quantification of first-order logic is extended
to allow for quantifiers to range over predicate and function symbols. 
A variable which can be substituted by a predicate symbol is \emph{higher-order}.
The set of higher-order variables is denoted as $\hovars$. 
A \emph{higher-order} term is a higher-order variable or a predicate symbol.
An atom which has at least one higher-order term is higher-order.
A definite clause with at least one higher-order atom is higher-order:

\begin{definition}
A \emph{higher-order definite clause} is of the form:
\[A_0 \leftarrow A_1, \ldots, A_m \]
where $m \geq 0$ and each $A_i$ is an atom of the form 
$p(t_1, \ldots, t_n)$, such that $p \in \predsig \cup \hovars$ and 
$t_i \in \consts \cup \predsig \cup \fovars \cup \hovars$. 
\end{definition}

\paragraph{Substitution}
Given a formula with variables $v_1, \ldots, v_n$, simultaneously replacing 
the variables with terms $t_1, \ldots, t_n$ is called a substitution. 
Such a substitution is denoted by
$\theta = \{v_1/t_1, \ldots, v_n/t_n\}$. A substitution $\theta$ unifies atoms 
$A$ and $B$ in the case $A\theta$ = $B\theta$. 

\section{Meta-Interpretive Learning}

This section starts by formally introducing the problem addressed by MIL.
To do so we first need the key MIL concept of metarules.
The second part of this section gives an overview of meta-interpreting
and how metarules lift this notion to meta-interpretive learning.

%From this point on we will start mixing the logical nomenclature with the
%corresponding notions in logic programming, always first introducing any such
%correspondence. 
%A \emph{definite program} is a collection of definite clauses, where all predicate
%symbols are ground.
%A logic program containing a higher-order clause is a higher-order logic program.
%A program \emph{entails} a (ground) atom if and only if the atom
%unifies with head of one of the clauses of the program,
%and a Prolog interpreter is able to use a resolution algorithm to compute a complete
%derivation down to asserted facts from this unified head.

\subsection{Problem Statement}
\label{sec:untyped-prob}

The user supplies a set of examples $\exs$ and background knowledge $\kb$.
All \emph{examples} $e \in \exs$ are ground atoms over the same predicate name.
The examples $\exs = (\exspos, \exsneg)$ are divided into positive 
and negative examples.
The \emph{background knowledge} $\kb  = \defc \cup \meta$ consists of
definite clauses $\defc$, representing program fragments, and metarules $\meta$. 
The definite clauses also encode the facts (clauses without body) that are asserted.

\begin{definition}
A higher-order formula of the form
\[ \exists \pi \forall \mu ~.~ A_0 \leftarrow A_1, \ldots, A_m \]
where $m \geq 0$, $\pi$ and $\mu$ are disjoint sets of higher-order variables,
is called a \emph{metarule}.
Each $A_i$ is an atom of the form $p(t_1, \ldots, t_n)$ such that 
$p \in \predsig \cup \pi \cup \mu$ and
each $t_i \in \consts \cup \predsig \cup \pi \cup \mu$.
\end{definition}

Metarules differ from higher-order definite 
clauses in that they allow existential quantification.

Table \ref{tab:metarules} lists common metarules, a selection of which will be
used throughout this document.
Note that we elide the quantifiers, e.g.~the full definition of
the Identity metarule is 
$\exists P \exists Q \forall A \forall B~ P(A,B) \ot Q(A,B)$.
When quantifiers are omitted, we always label universally quantified 
first-order variables with names from the start of the alphabet ($A,B,C,\ldots$)
and existentially quantified higher-order variables with names 
starting from $P$ on in the alphabet.

\begin{table}
\begin{center}
 \begin{tabular}{|l | l|}
\hline
\textbf{name} & \multicolumn{1}{|c|}{\textbf{metarule}} 
\\\hline
Identity &\hspace{0.5cm} $P(A,B) \ot Q(A,B)$
\\\hline
Precon &\hspace{0.5cm} $P(A,B) \ot Q(A),R(A,B)$
\\\hline
Curry &\hspace{0.5cm} $P(A,B) \ot P(A,B,R)$
\\\hline
Chain &\hspace{0.5cm} $P(A,B) \ot Q(A,C),R(C,B)$
\\\hline
Tailrec &\hspace{0.5cm} $P(A,B) \ot Q(A,C),P(A,B)$
\\\hline
\end{tabular}
\end{center}
\caption{Common metarules, where variables $A$,$B$, and $C$ are universally
quantified and $P$, $Q$, and $R$ are existentially quantified.}
\label{tab:metarules}
\end{table}

\begin{definition}
\label{def:untyped-cons-hypo}
Given $(\kb, \exs) = (\kb, (\exspos,\exsneg))$ as background knowledge and examples,
respectively, a program $H$ is a \emph{consistent hypothesis}, 
denoted $H \cup \kb \entails \exs$, if all positive examples are entailed
by the program, and none of the negative examples are.
%\begin{itemize}
    %%\item the program contains a clause with as head the predicate mentioned
    %%    by the ground atoms in $\exs$,
    %%\item the program's definite clauses all either occur directly in
    %%    the knowledge base $\kb$, or the clause's structure conforms to one 
    %%    of the metarules in $\kb$.
    %\item for each $e^+ \in \exspos$ the program entails the positive example
        %$e^+$, denoted $H\cup\kb \entails e^+$.
    %\item for each $e^- \in \exsneg$ the program does not entail the negative
        %example $e^-$, denoted $H\cup\kb \notentails e^-$.
%\end{itemize}
\end{definition}

\begin{definition}
A \emph{MIL learner} takes an input $(B, E)$ and either outputs a definite 
program  $H$ that is a consistent hypothesis for the input, or terminates
stating failure to find a program.
\end{definition}

We consider MIL learners that have the additional guarantee that they return
programs that are optimal in a textual complexity sense.
The optimizing criterion used is the number of clauses in the program.
The guarantee is that there is no other consistent hypothesis which has fewer clauses.

\subsection{Meta-Interpreting and Metarules}

A Prolog meta-interpreter evaluates a Prolog-like language by unifying a 
goal with a head of one of the first-order clauses that it has available.
The atoms in the body of the unified clause become new goals subject to
the same procedure. For example, SLD-resolution on the goal 
$ancestor(alice,charlie)$, given the definite clauses
\begin{align*}
&parent(alice,bob)\\
&parent(bob,charlie)\\
&ancestor(A,B)\ot parent(A,B)\\
&ancestor(A,B)\ot parent(C,B),ancestor(A,C)
\end{align*}
yields by resolution with the last clause (unifying the goal with the head of
this clause) that the goals required to prove
become $ancestor(alice,C)$ and $parent(C,charlie)$. To prove the first of
these two goals the first clause with $ancestor$ as head is chosen 
(non-deterministically out of the two options). The two goals
are now $parent(alice,C)$ and $parent(C,charlie)$. Unifying the first goal
with the first fact fixes $C$ to $bob$ allowing both goals to be proven
by the asserted facts.

%\rm{Might want to discuss the SLD-tree derivations at this point, with a
%figure}

The Meta-Interpretive Learning framework is a meta-interpreter that additionally
tries to unify a goal with the head of a metarule, which is
a higher-order clause.
Upon selecting a metarule to prove the goal, the unification 
is saved in the form of a \emph{meta-substitution}. 

\begin{definition}
Let $M$ be a metarule with the name $x$, $C$ be a horn clause,
$\theta$ be a unifying substitution of $M$ and $C$, and 
$\Sigma \subseteq \theta$ be the substitutions where the variables are all 
existentially quantified in $M$, such that
$\Sigma = \{v_1/t_1, \ldots , v_n/t_n\}$.
Then a \emph{meta-substitution} for $M$ and $C$ is an atom of the form:
$sub(x, [v_1,\ldots,v_n]\{v_1/t_1, \ldots , v_n/t_n\})$, where the second
argument is a list of
logical variables with the appropriate substitution applied.
\end{definition}

Saved meta-substitutions are reused for proving goals that are encountered
later on, becoming available as unification targets like the definite clauses 
in the background knowledge.
Upon completing the proof of all goals the meta-substitutions contain a description
of the program. To obtain the program the saved substitutions are applied
to the named metarules. Each such instantiated metarule corresponds to an
invented clause of the program, with the predicate symbols being grounded 
through the substitution.

\begin{definition}
Let $(\kb, \exs)$ be a MIL input and $H$ a hypothesis. 
Then a predicate $p$ is an \emph{invention} if it is in the predicate signature
of $H$ and not in the predicate signature of $\kb \cup \exs$.
\end{definition}

Metarules form the heart of the MIL approach to learning.
Metarules allow for introducing a strong bias to the hypothesis space. This
is due to the clauses of programs in the hypothesis having to conform to the
structure of the supplied metarules. Metarules hence give great control over
the size of the search space, as well as how the search space is traversed.
They also allow the user to specify which features they consider likely to be
needed by the program.
Examples are (tail-)recursion and higher-order clauses.

In recent work by \cite{Cropper2016} the MIL framework is expanded so as
to allow background predicates and inventions with higher-order arguments.

\section{\metaai{}: Untyped Abstracted MIL}

The implementation of MIL that we consider as a basis in this document is
the \metaai{} system introduced by \cite{Cropper2016}. This section shows how
to specify a problem instance in Prolog code and goes over a high-level version
of the algorithm.

The improvement of
the \metaai{} versus the original \metagol{} implementation of MIL \citep{Muggleton2014}
is in allowing higher-order predicates to be part of the program, in particular
predicates that are termed \emph{abstractions}.
As shown in \citep{Cropper2016}, learning higher-order inventions has benefits
in terms of finding smaller programs, which leads to reduced learning times,
and can achieve high accuracy with a reduced number of examples.

\begin{definition}
A \emph{higher-order definition} is a set of higher-order definite clauses 
with matching head predicates.
\end{definition}

The following definition of $map$, operating over lists\footnote{A list in Prolog is either $[]$, the empty list, or is a cons of a head $H$ and a tail $T$, denoted as $[H|T]$.}, is an example of a higher-order definition:
\begin{align*}
    &map([],[],F) \ot\\
    &map([A|As],[B|Bs],F) \ot F(A,B), map(As,Bs)
\end{align*}

Any clause which takes an argument that is a predicate is termed an
\emph{abstraction}:
\begin{definition}
A higher-order definite clause of the form 
\[
    \forall \tau ~ p(s_1, \ldots, s_m) \ot q(u_1,\ldots,u_n,v_1\ldots,v_o)
\]
where $o >0$, $\tau \subseteq \fovars \cup \hovars$, $p,q, v_1, \ldots, v_o \in \predsig$,
and $s_1, \ldots, s_m,u_1, \ldots, u_n \in \fovars$, is called an \emph{abstraction}.
\end{definition}

The following clause, which increases each element of a list by one, is an example 
of an abstraction:
\[
    increment\_all(A,B)\ot map(A,B,succ).
\]

\subsection{Representation of Background Knowledge}
\label{sec:untyped-prolog-repr}

For Prolog code we will use the \texttt{typewriter} font.
Definite clauses in Prolog are very similar to the logic syntax, except the
arrow ($\ot$) is replaced by an ASCII version (\texttt{:-}) and every clause
is terminated by a dot. For clauses with empty bodies the ersatz arrow
is omitted.

The background knowledge is separated into three parts: \emph{primitive clauses},
\emph{interpreted clauses} and metarules.
Primitive clauses are just standard Prolog definite clauses, as such they do not
involve atoms whose predicate symbol are variable.
Such clause definitions are added to the background knowledge by asserting that
they are available as a primitive, e.g.~\texttt{id(X,X).}~is a definite clause, and by
asserting \texttt{prim(id/2)} the predicate is added to the background knowledge.

The interpreted clauses are Prolog clauses that do involve body atoms whose
predicate symbols are existentially quantified, and in particular are arguments
to the head of the clause. Because during the synthesis algorithm a higher-order
argument might remain undetermined until the body is evaluated, the normal 
Prolog evaluation strategy is not sufficient (see the below algorithm).
These clauses are asserted to be \emph{interpreted}, e.g.~\texttt{interpreted(map/3)}.

Asserting \texttt{metarule(Name,Subs,(Head:-Body))} adds a metarule
with the name \texttt{Name} to the background knowledge.
\texttt{Subs} is a list of the existential (higher-order) variables
occuring in the metarule, \texttt{Head} is the head atom of the rule and 
\texttt{Body} is a list of the body atoms of the rule. 
Take the chain rule as an example:
\[
    \texttt{metarule(chain,[P,Q,R],(P(A,B):-[Q(A,B),R(B,C)])).}
\]

\subsection{The Algorithm}
\label{sec:metaai-algo}

Figure \ref{fig:metaai} contains the Prolog code for the abstracted MIL algorithm,
which has support for higher-order abstraction.

\begin{figure}[t]
\begin{center}
\begin{tabular}{c}
\begin{lstlisting}
learn(Pos,Neg,Prog):-
  prove(Pos,[],Prog),
  not(prove(Neg,Prog,Prog)).
prove([],Prog,Prog).
prove([Atom|Atoms],Prog1,Prog2):-
  prove_aux(Atom,Prog1,Prog3),
  prove(Atoms,Prog3,Prog2).
prove_aux(Atom,Prog,Prog):-
  prim(Atom),
  call(Atom).
prove_aux(Atom,Prog1,Prog2):-
  interpreted((Atom:-Body)),
  prove(Body,Prog1,Prog2).
prove_aux(Atom,Prog1,Prog2):-
  member(sub(Name,Subs),Prog1), 
  metarule(Name,Subs,(Atom:-Body)),
  prove(Body,Prog1,Prog2). 
prove_aux(Atom,Prog1,Prog2):-
  metarule(Name,Subs,(Atom:-Body)),
  prove(Body,[sub(Name,Subs)|Prog1],Prog2).
\end{lstlisting}
\end{tabular}
\end{center}
\caption{The \metaai{} algorithm.}
\label{fig:metaai}
\end{figure}

\paragraph{Invocation}
The first clause, \texttt{learn}, is the invocation point of the algorithm.
After the user has asserted their background knowledge, they run the algorithm
by calling \texttt{learn} with a list of their positive and a list of
their negative examples whereby, upon success, \texttt{Prog} gets instantiated
to a program.
The algorithm starts out with an empty program, the first argument
to \texttt{prove} in the body of \texttt{learn}. If a program is found entailing
the positive examples it is checked against the negative examples. Were one
of the negative examples to be entailed, backtracking occurs and 
the search will continue for another program entailing the positive examples.
When all the positive examples are entailed by the found program, and none of
the negative examples are, the search successfully terminates.

\paragraph{Search}
The \texttt{prove} procedure is mainly for choosing the first goal in a list
of goals to hand off to the \texttt{prove\_aux} procedure. The mutual recursion
of \texttt{prove} and \texttt{prove\_aux} represents a left-most
depth-first search over the goals that arise during synthesis.
When a goal has been proven successfully by \texttt{prove\_aux} it may have
changed the program under construction, hence the new program is passed
along for proving the remaining \texttt{Atoms} goals. As the last disjuctive
clause \pc{prove\_aux} will always succeed with introducing a new invented
clause, an additional mechanism, not shown in
the code, of an limit on the number of inventions is used.
The limit makes sure that the depth first search does not run off and keeps 
creating invented clauses for goals that are difficult (or impossible) to prove.
By re-running the algorithm with increasing invention limit all possible 
programs are considered.

\paragraph{Additional proving rules}
For the most important part of the algorithm, 
we consider each of the disjunctive bodies of \texttt{prove\_aux} separately:
\begin{itemize}
\item The first disjunct tries to prove an atom (whose predicate symbol might
    be a variable) by unifying the atom's predicate with a predicate asserted 
    as background knowledge, considering all options based on the predicate's arity.
    When this 
    unification succeeds the atom is unified with the head of the predicate,
    whereupon
    it is known that \texttt{Atom} represents a first-order atom which 
    Prolog is able to evaluate. Evaluation is
    invoked by \texttt{call(Atom)}. If successful the atom's predicate symbol
    will remain fixed for for any goals subsequently generated.
    If the call to \texttt{call}
    fails, or having tried all possible programs 
    which included this particular choice for the atom causing Prolog to 
    backtrack to this decision point,
    all remaining predicates in the background knowledge will be tried in the same
    manner. If none lead to a successful program being found this 
    disjunct of \texttt{prove\_aux} fails, causing the next disjunct to be tried.
\item The second disjunct tries to prove the goal atom by unifying the atom
    with the head of one the head of interpreted background predicates.
    Again, if the predicate symbol is a variable it will become fixed upon
    successful unification.
    By unifying the head of an interpreted clause the corresponding 
    variables in the \texttt{Body} change accordingly, thereby making the
    body of the interpreted clause the goals that need to be proved to be 
    able to conclude the unified head.
\item Upon failure to find a successful program by proving the atom by
    an interpreted predicate, the algorithm checks if it may make use of any
    of the invented clauses that it keeps track of in the form of meta-substitutions.
    For each meta-substitution that is already in the program the algorithm
    tries to unify the atom with a fresh head of the meta-substitution's metarule,
    under the restrictions of the variables that are already fixed in the substitution
    list \texttt{Subs}. If it succeeds, the body of the reused invented clause 
    represents the goals that need to be proven.
\item Finally, the last disjunct applies when all other disjuncts failed to
     lead to a successful program. Now a metarule is instantiated, with no
     conditions on the \texttt{Subs}titutions for the existentially quantified
     predicate symbols. This means that \texttt{Atom} is used to create a new 
     invention. The body atoms of the metarule become the goals to be proven and
     the invention is remembered by prepending a new meta-substitution to the
     existing program.
\end{itemize}

In overview: a depth-first search is conducted over partial programs, starting
with the examples as goals. To prove each goal, first it is attempted to be
proven by one of the primitive and interpreted predicates. In the first case
Prolog is able to evaluate whether the atom holds, and in the second case
the body of the higher-order predicate remain as goals to be proven by
meta-interpretation. Otherwise the saved invented clauses are tried, again
leading to additional goals to be proven. As a last resort an atom can be proven
by creating a metarule-guided invention, where the body of the invention will 
likely contain goal atoms with predicate variables.
If there are no more goals to prove a successful program has been found.

\chapter{Typed Meta-Interpretive Learning}

Starting with this chapter, this document will focus on presenting novel
contributions.
This chapter highlights the potential benefits of adding typing to  
Meta-Interpretive Learning (MIL) before the subsequent chapters focus on
tackling the search space reduction issue. After some preliminaries regarding
types, three areas are discussed where types bring with them great potential.
The first is the ability to prune away parts of the search space. Second, we 
consider the idea that types can provide guidance in choosing how to traverse
the search space. The final highlighted feature is that it might be possible 
to inform the user that the background knowledge that they provided is
insufficient for solving the synthesis problem.

\section{Typing Definitions}
\label{sec:typed-type-def}

As is usual in presenting logic, the definitions of the previous chapter did
not include any type decorations.
For our purposes a \emph{type} is a set of values. A value that belongs
to a type is called an \emph{inhabitant}.
Normally it is known for functions and predicates for which values a 
predicate or function makes sense.
We will restrict the domain of predicates/functions to just these values.
We will use such domain restrictions, along with co-domain restrictions,
in defining the formal types of predicates and functions. As a general
reference for refinement types with polymorphism see \citep{freeman1991refinement}.

\paragraph{Polymorphic types}
The types we consider are constructed from a collection of base types $B$,
among others, the integers, $int$, and the Booleans, denoted $bool$. 
We also allow \emph{holes} in our types, for which we use \emph{type variables}.
The Cartesian product ($\times$) is used to construct types whose values
are tuples of values from the supplied types.
The arrow ($\to$) is used to construct function types from other types.
We have that the grammar\footnote{This is in essence the type grammar of
System-F \citep{girard1971extension}.} for our types is stated as
\[ T ::=~b_i~\mid~X_j~\mid~T \times T~\mid~T \to T, \]
where $b_i \in B$ and $X_j$ represents type variables.
A type is \emph{higher-order} if a function type appears in an argument position.

\begin{definition}
A predicate $p$, with arity $n$, is a function whose result type is Boolean.
The \emph{predicate type} of $p$ is fixed by the types of its $n$ arguments.
This is denoted as:
\[ p : T_1 \times \ldots \times T_n \to bool \hspace{1em}\text{or}\hspace{1em}  p(a_1, \ldots, a_n) : T_1 \times \ldots \times T_n \to bool \]
where $a_i : T_i$ are the formal arguments of $p$.
%$p(t_1, \ldots, t_n) : Bool$ holds given the typing for the terms 
%$t_1 : T_1, \ldots t_n : T_n$. 
\end{definition}

We will abbreviate the predicate typing $p : T_1 \times \ldots \times T_n \to bool$
to the compacter $p : [T_1,\ldots,T_n]$. In what follows we will be lax in 
distinguishing between the type of an atom (always $bool$) and the type of
the atom's predicate. 

For function typing a similar definition holds, except that the co-domain
is not fixed to be Boolean. For functions we do not use shorthand notation.

A type is \emph{polymorphic} when it is parametrized by one or more types,
i.e.~it contains type variables.
We always consider type variables to be universally quantified.
For example, the polymorphic type of lists is $list(X)$.
Replacing the \emph{type parameter} $X$ with a type resolves the polymorphism, 
e.g.~take $X=int$, then $list(int)$ is the type of lists 
of integers, which is no longer polymorphic. 
A predicate is polymorphic when one of its argument types is.

\paragraph{Refinement types}
Given the above description of types we define the notion of
refinement type, and refer to unrefined types as \emph{simple (polymorphic) types}.
A \emph{refinement type} $\{ x : T \mid \phi\}$ is the subset of the 
type $T$ consisting of the values $x$ that satisfy the formula $\phi$. 
For this work we only allow refinements to occur on predicates.

In order to discuss the meaning of refinements on predicates, hereby
a reminder regarding the interpretation of predicates:
a predicate implements a relation, which identifies tuples, 
by mapping arguments included in the relation to $\top$ (true)
and otherwise returning $\bot$ (false).

Let $p(a_1, \ldots, a_n):T_1\times\ldots\times T_n\to bool$ be a predicate with its type.
A refinement $\phi$ for this type is a proposition that can mention any of the variables $a_i$.
The refined type is denoted as $T_1\times\ldots\times T_n\to bool\<\phi \>$, with
shorthand $[T_1,\ldots,T_n]\<\phi\>$.
The semantics of this type is that $\phi$ denotes a necessary 
condition of any $\top$ inhabitant of the Boolean result type.
In conventional refinement type notation this corresponds to 
\begin{align*}
    bool\< \phi \> &= \{ b: bool ~\vert~ b {=}\bot\vee (b {=} \top \implies \phi) \}\\
                &= \{ b: bool ~\vert~ b {=}\bot \vee (b{\neq} \top \vee \phi) \} \\
                &= \{ b: bool ~\vert~ \neg\phi \implies b {=} \bot) \}.
\end{align*}

Given a valuation for the arguments of $p$, the proposition $\phi$ has a truth value.
When a valuation makes $\phi$ false, the values of the arguments to the predicate 
cannot be identified by the relation that the predicate is representing, 
as the predicate must return false for these arguments. Hence refinement are
a way of restricting types to make them more accurate.

As an example consider the higher-order $map$ predicate, as of it is a good candidate
for a refinement type. 
The $map$ relation guarantees that the list arguments are of the same length.
This is encoded in a refinement type as follows:
\[ map(A,B,F):[list(T),list(S),[T,S]]\<length(A)=length(B)\>\]

It is important to note that simple polymorphic types are precisely our refinement
types with all refinements just being the $true$ proposition (not imposing
any restrictions). In this text we will elide types (and their refinements)
when they are not of interest to the discussion at hand.

\section{Pruning the Search Space}
\label{sec:typed-pruning}

The MIL approach is surprisingly effective gives its simplicity. 
The algorithm is quite naive in regard to what possibilities for programs
it is willing to consider.
One major shortcoming, inherited from its logic programming origins,
is its irreverence of types. Many of the logic programming languages do not
concern themselves with types. In particular, Prolog does not have predicate
types.

Consider having to prove a goal atom $P([true,false,true],[1])$, with
$P$ a predicate variable. We can immediately see that its type 
is $[list(bool),list(int)]$.
Let us suppose that the given background predicates include the predicates
$succ : [int,int]$ (the successor relation on integers),
$tail : [list(X),list(X)]$ (the list tail relation), and 
$map : [list(X),list(Y),[X,Y]]$ (the relation that maps another relation
over elements). The metarules we consider for this example are
Chain and Curry (see table \ref{tab:metarules}).
Clearly any predicate in the background knowledge over non-list types 
need not be tried as a direct substitution for $P$.
Also many inventions and interpreted predicates need not be tried, such as 
$inv(A,B)\ot tail(A,C),map(C,B,succ)$, purely due to the typing of this
clause being incompatible with the type of the example goal.

\subsection{A Worked Example}
The work involved in the above example is considerable given how easy it is to see,
based on typing, that it cannot succeed. First $P$ will be unified with $inv$
whereupon unification of the atoms assigns $A=[true,false,true]$ and $B=[1]$.
The body atoms $tail([true,false,true],C)$
and $map(C,[1],succ)$ become the new goals, being passed of to be proven
by a recursive \texttt{prove} call. 
The atom $tail([true,false,true],C)$ is proven in the first disjunct of 
\texttt{prove\_aux},
assigning $C=[false,true]$. Now the second goal $map([false,true],1,succ)$ is considered.
None of the primitives predicates unify with the predicate name, 
hence the interpreter clause of \texttt{prove\_aux} is evaluated.
Here the head of map is first checked for
whether it happens to unify with base case over empty lists, which it does not.
Subsequently the head of the inductive disjunct of $map$ is unified, whereupon
$succ(false,1)$ and $map([true],[],succ)$ become the new goals.
Again \texttt{prove} is called recursively, 
whereupon $succ(false,1)$ is selected to be proven by delegating the 
evaluation of the atom to Prolog itself.
Prolog determines that this atom does not hold.
It is only at this point that \metaai{} is able to decide
that the invented clause $inv(A,B)\ot tail(A,C),map(C,B,succ)$ could not 
be used, leading the meta-interpreter to backtrack to the decision point of 
having to pick a way of proving the original atom $P([true,false,true],[1])$.

\paragraph{Type Checked}
In contrast, when we consider types, a unification attempt is enough to
determine that the above work is unnecessary.
The type annotations for the atoms are
\[ P([true,false,true],[0]):[list(bool),list(int)] \] and
\[ inv(A,B):[list(int),list(int)]\ot tail(A,C),map(C,B,succ). \]
Type checking corresponds to checking whether the type of $inv$ can
 have zero or more of its variables unified such that it corresponds to
$P$'s type. Clearly the first arguments' types cause unification to fail,
thereby rejecting the attempt to try to prove the goal atom with the invented
clause.

This is just one example of where simple type checking helps considerably. 
Every time type checking can determine that one of the options considered by
\metaai{} cannot be successful exploration of a part of the search space 
can be skipped over, which is called \emph{pruning}. 
When such parts of the search space contain multiple decisions points considered
by the algorithm the benefit of type checking becomes considerable more
significant.

\subsection{Granularity of types}

Polymorphic types are already quite powerful in pruning parts of the search space.
But suppose we make a simple adjustment to the above example. Instead
of our goal's arguments including Booleans we change these values to integers,
making the example goal $P([1,0,1],[1])$. A type check determines that the atom's
type $[list(int),list(int)]$ is unifiable with $[list(int),list(int)]$,
hence giving the go-ahead to attempt to prove the atom with $inv$.

\paragraph{Unnecessary work revisited}

Because polymorphic type checking is unable to rule out the invented
clause the algorithm will be forced to do the same work as detailed above.
When it reaches the atom $succ(0,1)$ Prolog will instead determine that this
atom does hold. Now even more work will be performed, namely the remaining
goal, $map([1],[],succ)$, will be passed off to \texttt{prove\_aux},
which again checks that this is not a primitive before checking the interpreted
disjunctive clauses of $map$. Thanks to the first argument the base case
does not apply, and because the second argument does not allow itself to be
split into a head and tail the inductive case of $map$ also does not apply.
So in this case, there was not only more work in trying to prove additional 
goals before finally finding out that the invented clause cannot be used,
but we also incurred a small cost for the unification attempt of the type 
checking.

Intuitively, when we look at the example goal and at the invented clause we are
still able to immediately see that this clause cannot work. This is due to
being able to reason about the lengths of the lists involved. Clearly 
all the tuples entailed by the invented clause have as property that the 
first argument has exactly one additional element versus its 
second argument.

\paragraph{Refinement reasoning}

The above noted relation on the arguments of $inv$ is a property well captured
by a refinement type and can be stated as:
\[ inv(A,B):[list(int),list(int)]\<length(A) = length(B) + 1\> \]

This type states that the predicate can only entail the arguments if the
first list has a length exactly one longer than the second list.
If we now return to the type checking of $P([1,0,1],[1]):[list(int),list(int)]$
against the type of $inv$ we not only try to unify the polymorphic types, but
we also check that the refinement is still satisfiable, i.e.~does not preclude
the arguments from being entailed. What happens is that the heads are 
speculatively unified, leading to the refinement becoming instantiated to 
$length([1,0,1]) = length([1]) + 1$. For this example checking the
refinement can simply be done by expanding the definition
of $length$, which directly derives the contradiction $3 = 2$.
The refinement type hence declares that these values are never part of
the relation encoded by the $inv$ clause. At this point the algorithm will give
up on $inv$ and will try the next option for proving the example atom.
Refinement type checking is hence able to prune parts
of the search space that simple polymorphic type checking is not able to. 

There is, of course, the issue that reasoning over (instantiated) refinements 
is not entirely trivial. For the example the work needed to come to a contradiction
was very simple. The main challenge for the usage of refinement types is in
identifying refinements that suitably abstract from the predicates, and in
finding algorithms that are able to very rapidly reason on the constraints 
specified by the refinements.

\paragraph{Composing refinements}

The above identified refinement for $inv$ is not something a user would be
able to supply to the system, as the clause in question is an invention. The body
atoms on the other hand are background predicates which the user can
annotate with appropriate refinements. The refinement for $tail(A,C)$ could well
be 
\[ length(A) = length(C) + 1 \]
while $map(C,B,F)$'s refinement can be taken to be $length(C) = length(B)$.

The structure of the clause now indicates how these refinements compose,
namely $length(A) = length(C) + 1 \wedge length(C) = length(B)$.
Clearly, this approach to refinements is quite powerful in that is it able
to capture sensible refinements even for invented clauses.

\section{Types Directing the Search}
\label{sec:typed-direct-search}

The \metaai{} algorithm uses a basic depth-first search procedure for determining
the order in which goals are proven.
A novel observation is that this search might be better steered by the 
information available in the goals that remain to be proven.
The idea is that the ordering of the goals in the search needs to be guided by
a heuristic function. Correspondingly the search algorithm needs to be adjusted so as
to implement the best-first search procedure. Best-first search keeps track
of all the nodes left to explore, in our case the current goals left to prove,
and selects among them the best node to explore/prove next (according to the heuristic).
As it is non-obvious what the main features of interest are and 
how much weight they should have relative to each other, work is needed
in determining good heuristic ordering functions.

The goals are the basic units whose proof order needs to be decided on,
which, in the untyped case have quite limited information available.
Obvious considerations are to prioritize those goals that are already entirely
ground, including the predicate symbol(s). Amongst them the atoms with primitive
predicate symbols would need to be sorted first as an inconsistency on such a goal
would lead to immediate backtracking. Next one would sort on the 
number of variables in the goals, preferring atoms with fewer variables as
these are less likely to incur non-determinism in their proof.
Issues in ordering start to crop up when deciding on the ordering of atoms 
with different arities and whether the proportion of the number of ground 
arguments might not be a more useful feature than just the sheer number of 
them.

\paragraph{Type Guidance}
The introduction of types is interesting for heuristic search because it adds
additional information to use for determining the ordering.
As we saw in the section on pruning the search space, types of the predicates
are often already known while the arguments of such predicates are still variables.
During synthesis, e.g.~of a new invention, there are also variables in the types,
meaning that some types are more ``complete'' than others.

A heuristic in the typed environment could consider ordering goals with 
complete types (i.e. no variables in them) first, as these types 
will be able to rule out large parts of the background predicates.
Subsequently it would be useful to consider how to order the predicates
with incomplete types. To define a useful heuristic it would then be necessary
to weight the type features relative to the non-type features.

\paragraph{Implementation}
The accompanying synthesis system implemented for this project has best-first
search available as an option. There is a proof-of-concept heuristic 
that looks at type information and the number of ground arguments.
While the implementation is able to show some benefits in exploring less 
of the search space, the heuristic needs more thought before the value of 
heuristic search in MIL can be thoroughly evaluated.
For now we defer work on heuristics.

\section{Showing Insufficiency of the Background Knowledge}
\label{sec:typed-insuff}

A third area where type annotated predicates might be useful is in
determining that the predicates that the user supplied will never be able
to prove the goals. Clearly it would be very beneficial for the user to be aware
that the background knowledge that they are providing is insufficient for solving
the problem. Ideally the synthesis algorithm is able to detect, no matter
the size of the programs it will consider, whether there are any sensible programs 
at all in the hypothesis space.

Deciding on whether the background predicates can be composed is something
that type checking is already able to reason about. Hence the idea is to
leverage the user provided type annotations to check whether the background
predicates compose at all.

One simple approach is to take the type of the examples and for each metarule
try types for the body atoms. The types one would try are the ones on
the background predicates. Simple types or refinement type checking
might indicate that not a single sensible composition was found. Such a result
would only prove that no single clause programs can be constructed for the given
examples. If we relax the condition that the clause must  match with the type 
of the examples we could show that no invention exist that is composed 
from just background predicates.

The above reasoning is not strong enough to make claims about the non-existence of
programs that are only slightly more complex. The main issue is that the type
of an invention does not have to be the same as the type of any of the background
predicates. Determining what types can be generated for inventions would 
be one approach to trying to show insufficiency of the background knowledge.

In this thesis we limit ourselves to presenting the above argument for why 
types could be useful for showing insufficiency of the background knowledge, 
but will leave a solution to the problem for future work.

\chapter{Synthesis with Polymorphic Types}

This chapter introduces polymorphic types to Meta-Interpretive Learning (MIL). 
The main motivation for the introduction of types is so that type checking 
is able to prune parts of the search space.
For a motivating example please refer back to section \ref{sec:typed-pruning}.

We restate the problem that our system is able to address, noting the type
annotations we expect the user to provide.
We introduce the \metast{} algorithm, which extends the \metaai{}
algorithm with polymorphic type checking.
We look at how the type annotated background knowledge, along
with unification, can propagate types through to the newly derived goals.
In the section that follows we look at inferring polymorphic types,
i.e.~generate the most general type of each clause, and of the program itself.

Towards the end of the chapter we argue for the algorithm's correctness
in terms of being sound and complete (relative to \metaai{}).
We close out the chapter by a theory result regarding the influence of types
on the size of the search space, and perform experiments validating the work of
introducing polymorphic types.

\section{Revised Problem Statement}
\label{sec:simple-problem}

We assume familiarity with the synthesis problem statement in 
section \ref{sec:untyped-prob},
as well as with the way users provide background knowledge to 
the \metaai{} algorithm, as discussed in section \ref{sec:untyped-prolog-repr}.
Instead of reiterating the complete definitions we note the needed
adjustments to these definitions.

In addition to supplying examples, the user now supplies a type
for the examples, that is a single type that is consistent for all examples.
For the background knowledge we stipulate that each atom that can become a goal 
within the MIL algorithm is annotated with a (polymorphic) type.
For the primitive clauses only the head of the clause needs to be assigned a
type. As an example, the (primitive predicate) $head$, which returns the first
element of a list, was added to the untyped
background knowledge as \pc{prim(head/2)} (with the $2$ signifying the predicate's
arity). In the setting of typed synthesis predicates need to be asserted
with their type, that is, \pc{prim(head:[list(X),X])}
adds the $head$ relation to the (typed) background knowledge.

The interpreted predicates and the metarules need types for their head atoms
as well as for their body atoms. The variable names used for the types are
shared in a definition of a clause,
which is the main means of propagating type information.

For adding interpreted predicates to the untyped background knowledge,
the definition of the interpreted $map$ predicate was stated as a normal
Prolog definition, plus an \pc{interpreted} assertion:

\begin{center}
\begin{tabular}{c}
\begin{lstlisting}
map([],[],F).
map([A|S],[B|T],F):-F(A,B),map(S,T,F).
interpreted(map/3).
\end{lstlisting}
\end{tabular}
\end{center}

For the typed assertions of interpreted predicates, we need to 
annotate the atoms of the body of an interpreted clause in addition to the
head atom. Following the notation used for metarules, the definition of the 
clause is moved into the
\pc{interpreted} assertion\footnote{Note that we make heavy use of Prolog's square 
    bracket list notation here: to operate over values, as the data structure
containing the multiple body atoms, and as a convenient syntax for predicate types.}:
\begin{center}
\begin{tabular}{c}
\begin{lstlisting}
interpreted(map([],[],F):[list(X),list(Y),[X,Y]] :- []).
interpreted(map([A|S],[B|T],F):[list(X),list(Y),[X,Y]] :- 
             [F(A,B):[X,Y],map(S,T,F):[list(X),list(Y),[X,Y]]]).
\end{lstlisting}
\end{tabular}
\end{center}

Metarule definitions now include types on their atoms. 
For the untyped Chain metarule the assertion 
\pc{metarule(chain,[P,Q,R],(P(A,B):- [Q(A,C),R(C,B)]))} sufficed. 
For adding the Chain metarule to the typed background knowledge the assertion
becomes:

\begin{center}
\begin{tabular}{c}
\begin{lstlisting}
metarule(chain,[P,Q,R],
          (P(A,B):[X,Y] :- [Q(A,C):[X,Z],R(C,B):[Z,Y]]).
\end{lstlisting}
\end{tabular}
\end{center}

Definition \ref{def:untyped-cons-hypo}, on consistent hypotheses,
only needs adjustment in that the generated clauses of the program have
simple types on their atoms.

\begin{definition}
A \emph{Typed MIL learner} takes examples and 
background knowledge with polymorphic types and outputs a polymorphic
typed definite program $H$ that is a consistent hypothesis for the input.
\end{definition}

\section{\metast{}: Type Checking through Unification}

One of the major strengths of Prolog (and logic programming in general) is
support for logical variables and unification on these variables.
Unification plays an important role in most type checking and type inference 
algorithms (see, for example, \citep{kanellakis1989polymorphic}).
For our purposes we can reduce all type checking to unification. 

Unification on types serves two purposes during synthesis. 
For type checking we want to show
that a general (polymorphic) type can be instantiated to a
more specific type, e.g.~as
in the case of \pc{head:[list(X),X]}, and proving an atom \pc{P([1],1):[list(int),int]}.

At the same type we often have type variables in atoms' types due to the exploratory
nature of synthesis (e.g.~the \pc{Z} variable in the Chain metarule definition
of the previous section will always be just a variable right after the metarule is 
used for inventing a clause).
These type variables represent freedom in how to interpret 
the atom, in particular these variables represent that the type
of the arguments is as of yet undetermined.
Non-ground arguments (i.e.~arguments with variables in them)
will often initially have a variable for their type.
For example, suppose we are trying to prove atom \pc{P([1],B):[list(int),X]}.
Unification with the \pc{head} primitive on the predicate name and type fixes 
this atom's type  to \pc{[list(int),int]}.

In either case, whether we are type checking
or making types more specific, when unification succeeds we know we can proceed in
our proof attempt.

\subsection{Derivation and General Types}
The major change we introduce to the \metaai{} algorithm is 
that each atom has both a \emph{Derivation Type} (\pc{DT}), 
and a \emph{General Type} (\pc{GT}), that is \pc{Atom:DT:GT}.
The derivation type is always an instance of the general type of the atom,
which can be seen as instantiating any type parameters in the general type to
correspond to the types of the values that the arguments have taken on in the
derivation.

The derivation type is for keeping track of the type of the atom as 
it is used in the proof of the entailment for the
user provided examples. The derivation type hence is as accurate as possible
taking into account the \emph{values} that the atom's arguments have taken on.
The values an atom's arguments take on during the algorithm 
ultimately derive from the values of the user provided examples.
As such the derivation type can be seen as deriving
from the example goals given to the algorithm, which are the root of the
derivation tree (implicitly) constructed by the algorithm. The algorithm
will terminate when a successful complete derivation 
has been found for the example goals, 
at which point all arguments in the derivation tree will have been instantiated
to values. Therefore the derivation types for atoms in a successful derivation
are never polymorphic (contain no type parameters).

The general type, \pc{GT}, on the other hand is not concerned with being accurate
with regard to the values that the atom has taken on in the current derivation.
Instead the general type sees the atom as the head atom of a
definite clause and maintains the type that the atom's arguments 
\emph{may} be instantiated with, i.e.~its polymorphic type.
The general type is hence
determined by the constraints imposed on the arguments' types based on 
the types of the atoms that become goals to prove this atom.
%, i.e.~due
%to the atom being proved by an interpreted predicate or an invention.
The general types can hence be seen as deriving from the leafs in the
derivation, up through the subtree under this particular atom.
In particular, the general type is not influenced by the example goals 
(and type) given to the algorithm.
Upon a successful derivation being found, in general,  it is \emph{not} the 
case that the general type does not contain variables, 
instead the point of the general types (of the head atoms of inventions) 
is that they may be polymorphic.

\paragraph{An example}
As an example, suppose we have to prove \pc{P([1,2,3],2):DT:GT}.
Its derivation type will already be fixed as \pc{DT=[list(int),int]},
due to the algorithm being able to maintain fully accurate derivation types 
for arguments instantiated with values (given that the type of the example 
goals was provided and the background predicates had accurate types).
As we have not yet proven the atom the general type will usually be 
$GT=[X,Y]$.\footnote{
In the case that the argument variables are shared with other atoms in a clause
it might be that there are already constraints imposed on the general type
(which would necessarily also be present in the derivation types). 
Such constraints make the type more specific than just two variables.}
Suppose the invented clause the algorithm comes up with to 
prove the atom is \pc{inv(A,B):-tail(A,C),head(C,B)}.
We have that the (general) type of the body atoms is
\pc{tail:[list(X),list(X)]} and \pc{head:[list(Y),Y]}. The general type
of the \pc{inv} derives from these general types as being \pc{[list(X),X]}.

As part of proving the goal atom by this invention we unify the head of the
invention with the goal to obtain \pc{inv([1,2,3],2):[list(int),int]: [list(X),X]}.
Now the body atoms of the invention need to proven, which are 
\pc{tail([1,2,3],C):[list(int),list(int)]:[list(X),list(X)]} and atom
\pc{head(C,2):[list(int),int]:[list(Y),Y]}.
The derivation types have the correct type
for \pc{C} imposed by the typing of the background predicates of \pc{tail}
and \pc{head} forcing unification on \pc{C}'s type based on the type of the
other arguments. Clearly the proof succeeds when \pc{C=2} is forced by
proving the \pc{tail} atom.

This example makes the utility of having a general type for \pc{inv} apparent.
If we would need to prove another goal such as 
\pc{Q([[1],[2]],[2]]:
[list(list(int)),list(int)]:QGT}, we are not restricted to
just being able to retrieve the derivation type that was used in the proof
for the previous goal. Instead we can check that 
this derivation type is also an instantiation of the general type of \pc{inv},
thereby allowing reuse of the invented predicate.

\subsection{The Algorithm}
The polymorphic type checking contribution to MIL of this thesis has resulted in
the \metast{} algorithm in figure \ref{fig:metast}.
The code in this figure is the \metaai{} algorithm, except for the bold code
which achieves type checking. 

\begin{figure}
\begin{center}
\begin{tabular}{c}
\begin{lstlisting}
learn(Pos,Neg,(*\pc{\textbf{Type}}*),Prog):-
  (*\pc{\textbf{map(decorate\_types(Type),Pos,PosTyped),}}*)
  (*\pc{\textbf{map(decorate\_types(Type),Neg,NegTyped),}}*)
  prove((*\pc{\textbf{PosTyped}}*),[],Prog),
  not(prove((*\pc{\textbf{NegTyped}}*),Prog,Prog)).
prove([],Prog,Prog).
prove([Atom|Atoms],Prog1,Prog2):-
  prove_aux(Atom,Prog1,Prog3),
  prove(Atoms,Prog3,Prog2).
prove_aux(Atom(*\pc{\textbf{:DT:GT}}*),Prog,Prog):-
  prim(Atom(*\pc{\textbf{:DT}}*)),!,
  (*\pc{\textbf{prim(Atom:GT),}}*)
  call(Atom).
prove_aux(Atom(*\pc{\textbf{:DT:GT}}*),Prog1,Prog2):-
  interpreted(Atom(*\pc{\textbf{:DT:-BodyDT}}*)),
  (*\pc{\textbf{interpreted(Atom:GT:-BodyGT),}}*)
  (*\pc{\textbf{combine\_types(BodyDT,BodyGT,Body),}}*)
  prove(Body,Prog1,Prog2).
prove_aux(Atom(*\pc{\textbf{:DT:GT}}*),Prog1,Prog2):-
  member(sub(Name,(*\pc{\textbf{GTinv}}*),Subs),Prog1), 
  (*\pc{\textbf{instance\_of(DT,GTinv),}}*)
  (*\pc{\textbf{GT=GTinv,}}*)
  metarule(Name,Subs,(Atom(*\pc{\textbf{:DT:-BodyDT}}*))),
  (*\pc{\textbf{metarule(Name,Subs,(Atom:GT:-BodyGT)),}}*)
  (*\pc{\textbf{combine\_types(BodyDT,BodyGT,Body),}}*)
  prove(Body,Prog1,Prog2). 
prove_aux(Atom(*\pc{\textbf{:DT:GT}}*),Prog1,Prog2):-
  metarule(Name,Subs,(Atom(*\pc{\textbf{:DT:-BodyDT)}}*)),
  (*\pc{\textbf{metarule(Name,Subs,(Atom:GT:-BodyGT)),}}*)
  (*\pc{\textbf{combine\_types(BodyDT,BodyGT,Body),}}*)
  prove(Body,[sub(Name,(*\pc{\textbf{GT,}}*)Subs)|Prog1],Prog2).
\end{lstlisting}
\end{tabular}
\end{center}
\caption{The \metast{} algorithm.}
\label{fig:metast}
\end{figure}

We have as main invariant of the algorithm that every 
goal provided as the first argument to \pc{prove\_aux} has its derivation
type \pc{DT} instantiated to the most specific type that is consistent
with how the atom was derived from the examples (using the program built up at 
to that point), and that \pc{GT} is the most general type of 
the atom as is constrained by the typing of the background predicates and
the invented clauses in the program. An atom's derivation type is always
an instance of its general type, as can be easily checked.

\subsection{Forward Propagating Derivation Types}

As part of specifying what problem the user would like to solve, the 
user supplies a type, \pc{Type}, for the examples.
In the \pc{learn} clause of \metast{},
the \pc{Pos} and \pc{Neg} examples are of the form \pc{p(a,\ldots,z)},
just as in the untyped case. 
The first two lines of the new \pc{learn} body make sure that the example
atoms satisfy the above invariant, namely each example is mapped from 
\pc{p(a,\ldots,z)} to \pc{p(a,\ldots,z):Type:GT}, where \pc{Type} is the
user supplied type, and \pc{GT} is a new entirely unconstrained type variable.

The \pc{prove} clauses of \metast{} are taken unchanged from the untyped algorithm.
Their main purpose is to implement (leftmost) depth-first search.

For this subsection we will focus on how the derivation type \pc{DT} is used
to prune parts of the search space and how the correct derivation type is
assigned to new goals. For now it suffices to know that the difference between 
the body atoms \pc{BodyDT}, the body atoms with derivation type, 
and the body atoms \pc{Body} in the \pc{prove\_aux}
clauses in \ref{fig:metast} is that \pc{Body}'s atoms additionally have their
general type set.

In discussing the \pc{prove\_aux} clause disjuncts we maintain the invariant
that the goal atoms, i.e~the first argument to the clause and the \pc{Body}
atoms, 
have the derivation type \pc{DT} instantiated to the most accurate known type.

\paragraph{Disjuncts}
In the first disjunct of \pc{prove\_aux} we check whether the atom can be proved
by one of the primitive background predicates. An atom's predicate will now
be matched against primitives, not only on arity, but also on the \pc{DT} type.
A primitive can only be chosen when the derivation type is an instance of the 
primitive's type. For example, take the goal \pc{P([1,2],B)} with
derivation type \pc{[list(int),int]}. When considering the \pc{tail:[list(X),list(X)]}
predicate, unification of the types fails due to \pc{int} and \pc{list(X)}
not being unifiable. A predicate such as \pc{head:[list(X),X]} does pass
the unification test due to \pc{X=int} equalizing the types.
Hence accurate derivation type checking occurs at the \pc{prim(Atom:DT)} line,
and failure to unify the types means that Prolog's evaluation will not be invoked
for the atom.

The second disjunct matches against the type annotated interpreted background
clauses. The assertion \pc{interpreted(Atom:DT:-BodyDT)} says to find a 
background clause definition such that the head of the clause successfully
unifies both the \pc{Atom}, containing a predicate name and values, as well
as the derivation type \pc{DT} with the clause's head and type. When the head unifies
successfully, the definition of the interpreted clause makes sure
that types of the body atoms are appropriately unified as well. 
Upon successful selection of a background predicate, the general type
of each atom is appended to their derivation type.

The third disjunct tries to reuse an invented clause of the program.
Here we use that the polymorphic type of (the head of) the clause is saved
along with the meta-substitution and the name of the metarule used for the invention.
The \pc{instance\_of(DT,GTinv)} goal asserts that the derivation type is
an instance of \pc{GTinv}. Note that \pc{DT} is not unified with \pc{GTinv} 
itself, which would make the type to specific, instead \pc{DT} is unified
with a copy of \pc{GTinv}.
We create a new instance of the invented clause, making sure that
the same meta-substitutions constraints are imposed,
and generate body atoms with the appropriate derivation types. 

The final disjunct deals with the option of inventing a new clause. The only major
change to note: in unifying the head of the metarule
with the atom we need to prove, the derivation type (and the general 
polymorphic type) are unified as well.

\subsection{Inferring the Most General Type}

Having looked at how goals get the correct derivation type assigned,
we now look at how the general type can be determined for atoms.
Note that determining the general type serves two purposes. First, we need the 
general type of an invention if we are to reuse an invention wherever we can,
e.g.~an invention \pc{inv(A,B):-tail(A,C),head(C,B)} derived to prove
\pc{P([1,2],2)} has derivation type \pc{[list(int),int]} when invented
for these values.
Instead we want to know the more general polymorphic type,
\pc{[list(X),X]}, such that the invention can be reused, e.g.~to
prove \pc{P([[],[1]],[1]):
[list(list(int)),list(int)]}.
Second, which follows directly from knowing the general type of inventions,
is that we can give a general type to the entire program, i.e.~to the head
of the clause used for the examples. This means we can also generalize from
the example type that we have been given.

All of the background knowledge (primitive and interpreted predicates, and
metarules) is already annotated with general types (and become more accurate
for goal derivation types by making them more specific through unification).
This is the property exploited in the algorithm, and why we see unification
on the background knowledge twice, once for the derivation types and once
for the general types.

\paragraph{Disjuncts revisited}
For the first disjunct it should be clear that unifying \pc{GT} with the
general type of the already selected primitive keeps track of the general
type of the atom. For the interpreted clauses of the second disjunct this holds
true as well, except that body atoms of the interpreted clause become goals
which can restrict the general type during their proof. We have that
background knowledge is only annotated with a single type,
which is why we unify twice, once we get body atoms with derivation types and
once we get body atoms with general types. 
The \pc{combine\_types} predicate combines these lists of singly typed
atoms, of the form \pc{BodyAtom:DTy} and \pc{BodyAtom:GTy}, into atoms with both
a derivation and general type, e.g.~\pc{BodyAtom:DTy:GTy}.

For the third disjunct we already explained how the general type of an invention
is used to type check. To keep track of the general type, note that the general 
type of the (head of the) invention must be the general type of the atom,
which is asserted by the equality.
This equality also makes sure that the general type among all usages of the 
invention remain consistent for this one general type.
The metarule instantiation with the general type \pc{GT} makes
sure constraints imposed by a new goal are directly reflected in the general
type of the invention.

The final disjunct, handling new inventions, is similar in reasoning to
the previous disjunct. The one thing to note is that because the new
invention is saved as part of the program, the general type for the
clause is stored with it. This is the general polymorphic type that will be shared
amongst all uses of the invention\footnote{An alternative 
simple approach would be to modify the definition of meta-subtitutions
to include types on the predicate names, forgoing the need to store the type
separately.}.

\section{Theoretical Results}
\label{sec:simple_theory}

To argue the correctness of the \metast{} algorithm we establish soundness,
i.e.~the programs found by the algorithm are correct for the examples,
and relative completeness, that is, every program found by the \metaai{} 
algorithm will also be found by \metast{}. Related to the completeness
result, we also briefly look at how sound pruning impacts predictive accuracy.
As a separate result we characterize how types on predicates can bound the 
size of the search space.

\subsection{Soundness}

\begin{definition}
An inductive synthesis algorithm is \emph{sound} if 
every program returned by the algorithm is a 
consistent hypothesis (definition \ref{def:untyped-cons-hypo}).
\end{definition}

To establish soundness of \metast{}, we make the following assumption:
the \metaai{} algorithm is sound.
Though we assume it here, it is not too hard to be convinced that this
assumption is true. In essence the meta-interpreter extends the proof procedure
of SLD-resolution with additional higher-order rules, and at the same time
maintains the proof steps in the derivation of the entailment of the examples
that an implementation of SLD-resolution does not itself maintain.

\begin{proposition}
Given that the \metaai{} algorithm is sound, the \metast{} algorithm 
is sound.
\end{proposition}

\begin{proof}
Assume the precondition.

Note that in logic/constraint programming we have that adding more 
constraints to clauses can only reduce the number of solutions found.
For the \metast{} algorithm to succeed the same derivation that the 
\metaai{} algorithm establishes must be found, i.e.~the derivation on
atoms without their types is exactly a \metaai{} derivation. This
is due to the \metast{} algorithm only adding constraints to the 
\metaai{} algorithm.

Clearly the type checking of \metast{} only imposes additional 
conditions on the proof of atoms, hence the derivations found by \metast{}
must be a subset of the derivations found by \metaai{}. Conclude
that any returned program by \metaai{} must as well be a consistent hypothesis.
\end{proof}

The only thing to remark with regard to the additional constraints imposed by
\metast{} is that they are not as simple as just additional atoms in bodies.
The head atoms gain some freedom in that they have additional variables in
their types, and some of the body atoms in the clauses are modified instead
of just added. These changes do not matter, as in the end the algorithm 
can only succeed when the atoms (without their types) form a proper \metaai{}
proof.

%The soundness result comes in two variants, distinguished by the assumption
%of whether the background knowledge is correctly annotated with types.

%\begin{proposition}{\bfseries Soundness for untyped programs}
%Suppose the \metaai{} algorithm is sound, then all \emph{untyped programs}
%(i.e. type erased programs) found by the \metast{} algorithm are consistent untyped
%hypotheses.
%\end{proposition}

\subsection{Completeness}

For our completeness result we assume that the background knowledge is 
typable by the polymorphic types introduced in section \ref{sec:typed-type-def}.
We also assume that the provided background knowledge has correct types, 
i.e.~the types might be too general, but they are not inaccurate.
For completeness relative to the \metaai{} algorithm we show that when the
\metaai{} algorithm is able to find a program, then \metast{} must find the
same program (syntactically identical).

Remember, from section \ref{sec:metaai-algo}, that the search procedure of 
the \metaai{} (and hence of the \metast{}) algorithm is leftmost depth-first.

\begin{definition}
A procedure that discards parts of a search space performs
\emph{inconsistency pruning} when
the parts of the search space pruned away never contain any successful nodes.
\end{definition}

\begin{proposition}
Let $A$ be a depth-first search algorithm.
If an algorithm $A'$ is algorithm $A$ except that it does additional 
inconsistency pruning,
then when algorithm $A$ finds its first successful node $s$, algorithm $A'$
will also find node $s$ as its first successful node.
\end{proposition}

\begin{proof}
Assume the stated relationship between $A$ and $A'$ and that algorithm
$A$ has found its first successful node $s$.
Suppose that either $A'$ does not find a successful node, or finds a node $t$.
Due to sound pruning we have that $A'$ cannot have pruned away the part
of the space containing $s$, and must come across it eventually.
Hence $A'$ must find node $t$. Because the traversal order of the nodes
is the same (modulo unsuccessful nodes being left out) it follows that $t = s$.
\end{proof}

First note that a program with a type check error cannot be a consistent
hypothesis (given at least one positive example).
That \metast{} adds inconsistency pruning relative to \metaai{} follows from 
that a type checking failure leads to pruning
and that any program with a type check error cannot be made to type check
by adding additional clauses to the program (i.e.~continuing the search).

That \metast{}'s type checking through unification correctly implements 
type checking can be established by induction on the size of the derivation
tree constructed during the algorithm\footnote{Note that we plan to make
the presentation of the type system formal such that this proof, amongst
others, can be proven formally.}.
The induction hypothesis is that the values are always inhabitants
of the derivation types. When the algorithm tries to prove a goal
with a primitive/interpreted/invented predicate and the typing of the predicate
cannot be instantiated to the derivation type of goal,
that predicate will be rejected due to a type error. 
This rejection is sound due to the predicate's type
indicating that the predicate cannot successfully the values of this type.
Otherwise, the predicate may be used to continue the search, at which
point it is easy to show that the values in the new goals are again
inhabitants of their type.

Using the above result we can immediately conclude the below proposition, by
instantiating the algorithms $A$ and $A'$ by \metaai{} and 
\metast{}, respectively.

\begin{proposition}
The (simply typable) programs found by the \metaai{} algorithm and the
\metast{} algorithm are exactly the same syntactically.
\end{proposition}

We will use this result in justifying why it is not worthwhile to experimentally
investigate the difference in accuracy of \metast{} and \metaai{}.
The above result only holds when the algorithms are allowed to run arbitrarily
long, long enough to find the (first) successful program. 
In practice a timeout is used. As we will see in the experimental work,
the untyped system is considerably slower than the typed system.
A consequence is that the typed system can find programs for which the
untyped system will use more time than is allowed by a timeout.
In such cases the predictive accuracy of the typed system will vastly
outperform the untyped system. In experiments where both systems are able to 
find the program the predictive accuracy is not impacted by type checking.

\subsection{Proportion of Relevant Predicates}
For the remainder of the theoretical examination we focus on a particular
language class of logic programs. The metarules in table \ref{tab:metarules}, e.g.~containing the
Chain rule $P(A,B) \ot Q(A,C),R(C,B)$, focus on restricting the structure
of the class of programs where each clause has a head atom and two body atoms.
Each atom's predicate has exactly arity two. This class is known as $H_2^2$.

The number of programs in this class is given as $O(|M|^n(p^3)^n)$, where
$|M|$ is the number of metarules, $n$ is the number of clauses and $p$ is the
number of predicate symbols \citep{lin2014bias}. The $p^3$ term is due
to that all three of the predicates in a clause being chosen independently from
each other.
If higher-order abstractions, each with $k \geq 1$ higher-order variables are
considered, this result is updated to $O(|M|^np^{(2+k)n})$ \citep{Cropper2016}.

We will suppose that the predicates in the background
knowledge are annotated with polymorphic types.
The improvement of taking types into consideration is due to predicate typing
having to match up, meaning that predicates in a clause usually cannot
be chosen independently from each other. 
Only a portion of the background predicates' types will line up. 

Given $p$ predicates with types, fixing one of the three predicates
in a $\Htwo$ clause restricts the choice of the two other predicates of the clause.
It hence makes sense to consider the maximum number of predicates that remain as
possible choices for any predicate of a $\Htwo$ clause,
after either (or both) of the other predicates have been selected. 
This is determined per instance of the background knowledge.

Let $\overline{p} \leq p$ be the worst case number of predicates that remain as choices for 
any of the predicates in a $\Htwo$ clause.
This value is determined by an exhaustive search
over the three predicate places in a clause, filling in any one predicate
and checking how many of the predicates can still be substituted for the
remaining predicate variables.
\begin{definition}
Given $p$ typed predicates, 
with $\overline{p} \leq p$ an upper bound on the number of typed
predicates that can be filled in any of the predicate variables of a $\Htwo$ clause,
given that another predicate of the clause has already been selected,
$t = \overline{p} /p$ is the \emph{worst case proportional constant}.
\end{definition}

The ratio will always be between $0$ and $1$, with lower values indicating
a greater reduction in the search space.
The proportional constant is a convenient value to work with due to it 
abstracting away the number of predicates in the program. The following
result characterizes the reduction in the hypothesis space of programs given
that predicates are properly annotated with types.

\begin{proposition}
\label{prop:space_red}
Given $p$ typed predicates, and the worst case proportional constant $t$ for the 
predicates, the hypothesis space $\Htwo$ is reduced by a factor of $t^{3n}$,
where $n$ is the number of clauses, versus the untyped hypothesis space.
\end{proposition}

The $p$ term in $O(|M|^n(p^3)^n)$ is the (maximum) number of predicates that can 
be filled in for any predicate variables in a untyped (unabstracted) clause.
For the typed case we know that this maximum is $\overline{p}$, and hence we substitute
$\overline{p}$ for $p$ in the size bound:
\begin{align*}
    |M|^n(\overline{p}^3)^n
    &=
    |M|^n((tp)^3)^n
\\  &=
    |M|^n(t^{3n})(p^{3n})
\\  &=
    (t^{3n})(|M|^n(p^3)^n))
\\
\end{align*}

In case of the abstracted search space we have that the reduction factor
is $t^{(2+k)n}$, using the exact same reasoning. These results imply
that using types to prune the search space leads to a considerable reduction 
in effort.

%\rm{TODO: can also use this result for sample complexity, where $ln(t)$ contributes
%a large negative term for small $t$. Andrew recommends working on this further. Blumer bound.}

\section{Experimental Results}
\label{sec:simple-expr}

There are two main concerns when evaluating synthesis systems: the speed with
which they are able to find programs and, in an inductive setting,
the predictive accuracy of the found programs for the relation that is being learned.

As shown in the previous section, when the search algorithm is (leftmost) 
depth-first search, the first encountered program that correctly entails the examples 
is always the same.
It follows that the accuracy of the found programs (barring timeouts, which
we will not consider in this section) is not affected by
polymorphic type checking based pruning, justifying the decision of only performing
experiments that check the impact on the size of the search space and 
the impact on the time needed for synthesis.

We perform three experiments to evaluate the benefit of polymorphic 
types to MIL. First, the \emph{Search Space Reduction} experiment checks that
the inclusion of irrelevant background predicates has negative effects for
the untyped framework, though the typed system is able to ignore them.
Second, the \emph{Ratio Influence} checks that the implementation is able to come close
to the theoretical result regarding better than linear influence of 
the ratio on the search space. 
Finally, we do a statistical experiment, \emph{Simply Typed Droplasts}, 
on the synthesis of the $droplasts$ program, checking for a time speedup and
a reduction in the number of derivation steps.

For simplicity's sake we will compare \metast{} to the untyped system by just
disabling \metast{}'s type checking\footnote{Time constraints caused us to
    decide that proper experiments generating both typed and untyped experiments
will have to be deferred to future work.}.
Note that this means that there is
some additional overhead versus \metaai{}, which does not keep track of types
at all. The number of proof steps is not impacted by this overhead, only the 
time needed is (though the impact should be quite small).

The Prolog implementation used for running the experiments is SWI-Prolog.

\subsection{Derivation steps}

In determining how efficient program synthesis is, we propose to keep track
of the number of decisions made for traversing the program search space.
That is, a decision is a choice made to construct part of a potential program.
A decision in the MIL framework corresponds to 
\begin{itemize}
\item trying to prove an atom with a primitive predicate,
\item trying to prove an interpreted higher-order predicate, 
by expanding the body,
\item trying an existing invented predicate/clause,
\item or choosing to apply a metarule, creating an invention. 
\end{itemize}

Figure \ref{fig:counter} indicates with bold lines where in the 
\metast{} algorithm the (global) decision counter is increased. 
Note, that for our purposes, a decision is made 
\emph{after} type-checking. For the untyped algorithm the checks occur at the
same place (modulo additional type checking lines).

Due to the correspondence of synthesis with constructing proofs, 
in particular derivations for the positive examples, the decisions will also 
be called proof steps or derivation steps.

\begin{figure}
\begin{center}
\begin{tabular}{c}
\begin{lstlisting}
prove_aux(Atom:DT:GT,Prog,Prog):-
  prim(Atom:DT),!,
  prim(Atom:GT),
  (*\texttt{\textbf{increase\_counter,}}*)
  call(Atom).
prove_aux(Atom:DT:GT,Prog1,Prog2):-
  interpreted(Atom:DT:-BodyDT),
  interpreted(Atom:GT:-BodyGT),
  combine_types(BodyDT,BodyGT,Body),
  (*\texttt{\textbf{increase\_counter,}}*)
  prove(Body,Prog1,Prog2).
prove_aux(Atom:DT:GT,Prog1,Prog2):-
  member(sub(Name,GTinv,Subs),Prog1), 
  instance_of(DT,GTinv),
  GT=GTinv,
  metarule(Name,Subs,(Atom:DT:-BodyDT)),
  metarule(Name,Subs,(Atom:GT:-BodyGT)),
  combine_types(BodyDT,BodyGT,Body),
  (*\texttt{\textbf{increase\_counter,}}*)
  prove(Body,Prog1,Prog2). 
prove_aux(Atom:DT:GT,Prog1,Prog2):-
  metarule(Name,Subs,(Atom:DT:-BodyDT)),
  metarule(Name,Subs,(Atom:GT:-BodyGT)),
  combine_types(BodyDT,BodyGT,Body),
  (*\texttt{\textbf{increase\_counter,}}*)
  prove(Body,[sub(Name,GT,Subs)|Prog1],Prog2).
\end{lstlisting}
\end{tabular}
\end{center}
\caption{\metast{}'s \texttt{prove\_aux} annotated with decision counter.}
\label{fig:counter}
\end{figure}

\stepcounter{expcount}
\subsection{Experiment \arabic{expcount}: Search Space Reduction}

In this experiment we verify, via deterministic tests, that 1) adding typed
predicates that do not compose w.r.t.~the input type are ignored by the
typed system, and 2) that the difference in time and proof steps 
corresponds to the theoretical results in proposition \ref{prop:space_red}.

\begin{hypothesis}
\label{hypo:mismatched-ignore}
\metast{} with type checking enabled is not able to prune the search space
relative to \metaai{} (i.e.~\metast{} with type checking disabled).
\end{hypothesis}

\begin{hypothesis}
\label{hypo:mismatched-pref}
\metast{} with type checking enabled is not able to learn faster than
\metaai{}.
\end{hypothesis}

\paragraph{Setup}
We ask the system to prove a single positive example \texttt{p(0,1)}, of type
$[int,int]$. In order to strictly control the search space we only provide the
chain metarule. As we want to traverse the entire search space we will only
introduce predicates that cannot contribute to a successful program. 

The predicate(s) in the background knowledge need to be general enough to
always succeed in unifying with goal atom (thereby making sure
that the largest possible search space is traversed).
The following predicate was chosen due to it being well-behaved when
the search space is traversed depth first:
\[ to\_zero(X,0). \]
We add one instance of this predicate with type $[int,int]$, thereby allowing
it to be used by the typed system. To test how the system handles completely
irrelevant background predicates we iteratively add additional instances of 
this predicate, but now with type $[bottom,bottom]$, where $bottom$ is
a dummy type. Predicates with both $bottom$ and $int$ types cannot
occur in the body of the Chain metarule.

\paragraph{Result}
The graph on top in figure \ref{fig:space_red_3clause} shows the number 
proof steps that were needed to traverse the search space, when the number of 
clauses in the largest program considered is restricted to three.
The graph on the bottom in this figure shows the time spent on traversing 
the search space in both the typed and the untyped system.

\begin{figure}[H]
\begin{subfigure}{\textwidth}
\centering
\includegraphics[width=\textwidth]{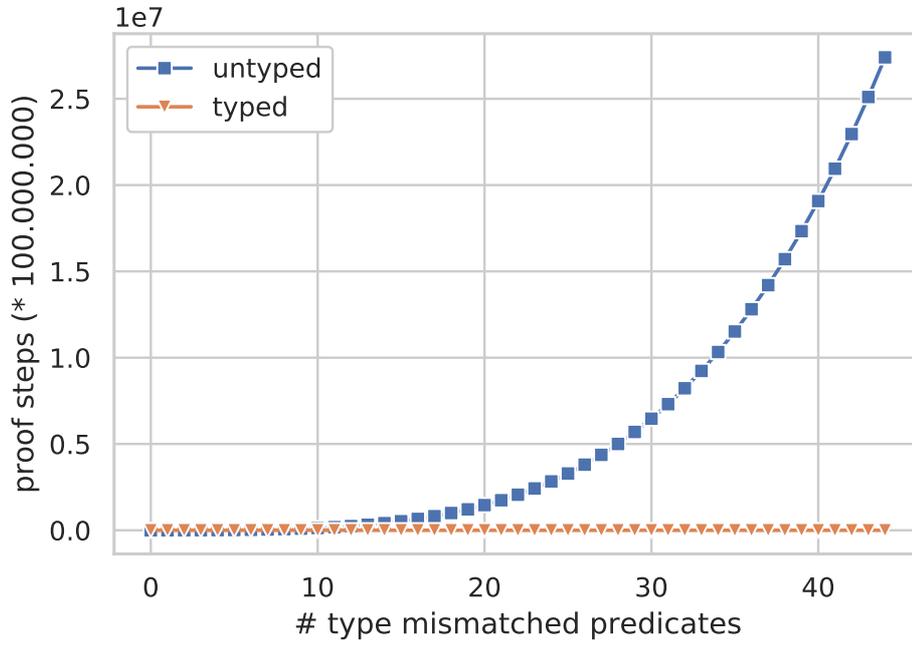}
\end{subfigure}

\begin{subfigure}{\textwidth}
\centering
\includegraphics[width=\textwidth]{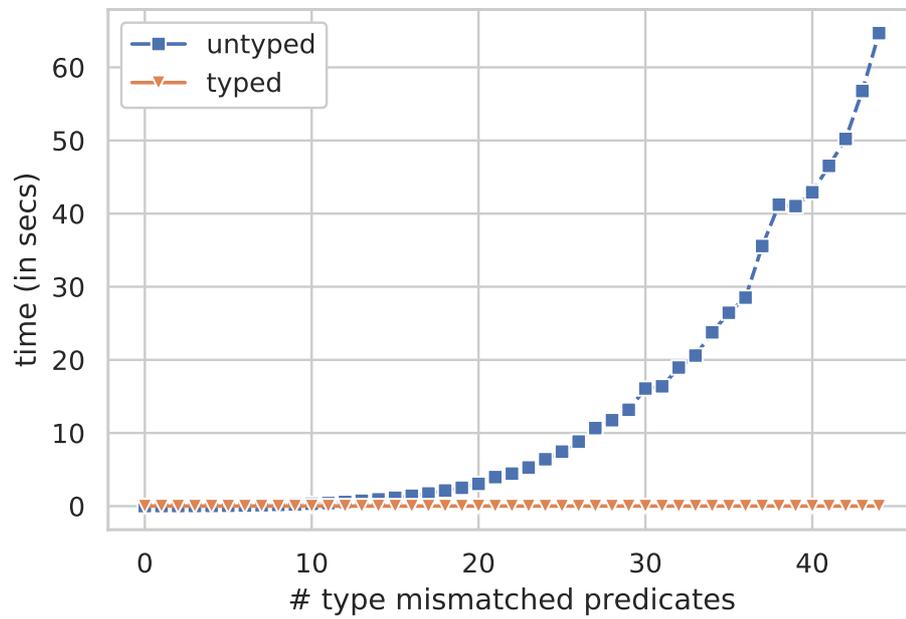}
\end{subfigure}
\caption{Number of proof steps and time for increasing number of mismatched predicates.
Max clauses is three. Typed proof steps are constant at $83$ and the time
required for typed synthesis is (approximately) constant at $0.014$ seconds.}
\label{fig:space_red_3clause}
\end{figure}

%\begin{figure}[h]
%\begin{subfigure}{.5\textwidth}
%\centering
%\includegraphics[width=\textwidth]{experiments/search_space_reduction/plot_time_4.eps}
%\end{subfigure}
%\begin{subfigure}{.5\textwidth}
%\centering
%\includegraphics[width=\textwidth]{experiments/search_space_reduction/plot_steps_4.eps}
%\end{subfigure}
%\caption{Time and proof steps for increasing number of mismatched predicates.
%Max clauses is four.}
%\label{fig:space_red_4clause}
%\end{figure}

The number of steps required by the typed system remain constant at $83$.
Correspondingly, the time for the typed program is also as near to constant,
never needing more than $0.014$ seconds. The untyped program on the other hand
traverses bigger and bigger search spaces with increasing number of (irrelevant)
background predicates. These graphs are strong evidence for rejecting 
hypothesis \ref{hypo:mismatched-ignore}, regarding not being able to further
prune the search space,
and for rejecting hypothesis \ref{hypo:mismatched-pref}, regarding performance
not being affected relative to the untyped system.

\paragraph{Additional Analysis} It is interesting to note that the graph for the proof
steps (and time as well), for a maximum of three clauses, is best explained\footnote{By
    running linear regression and determining that terms of higher degree
always are given trivially sized coefficients} by a polynomial of degree four.
This deviates from the known 
theory result regarding the size of the search space being $O(|M|^n(p^3)^n)$,
where $|M| = 1$, $n = 3$, and $p$ is varied in the experiment. Based on the
theory one would expect a polynomials of degree $9$ to best explain the
experimental data. The experiment was repeated to try to better understand
this observation: the number of clauses was limited to four, which resulted
in a polynomial of degree five best fitting the data, while a degree $12$
polynomial would be expected according to the theory.

Actually it is not hard to see that the theory result does
not take into account certain conditions on the predicates imposed by
the MIL framework. It only represents the size of the search space
traversed by a completely naive brute-force algorithm.
One clear example is that the predicates in the head
are restricted to the predicate of the examples and to invented symbols. 
It follows that at least one of the three predicates in a clause is restricted 
to $n$, instead of $p$. This observation allows us to replace the $p^3$ 
in the size expression with $p^2$ (given $p \gg n$).
This restricted size bound still does not fully explain our lower degree polynomials,
meaning that we need to refine the theory to accurately
capture the hypothesis space that is actually considered by the MIL framework.

\stepcounter{expcount}
\subsection{Experiment \arabic{expcount}: Ratio Influence}

For this experiment we consider the worst case proportional constant 
$t = \overline{p}/p$ as a varying ratio, where $p$ is the total number
of background knowledge predicates, of which at most $\overline{p}$
of which can be used for any predicate variable.
Via a deterministic test we determine the influence of the ratio of matching types
predicates versus all background predicates on the search space that
the \emph{typed} system, \metast{}, explores. Again, we are interested in the 
size of the search space, hence we do not try to find a successful program.

\begin{hypothesis}{Ratio unimportant}
\label{hypo:ratio}
The ratio $t$ does not influence the size of the search space explored.
\end{hypothesis}

\paragraph{Setup}
We take as basis the previous experiment.
Again, we take the $to\_zero(X,0)$ predicate, though this time we insert $25$ of
them in the background knowledge. We vary the types of the predicates
such that different ratios are achieved, e.g.~$1/25$, $2/25$, etc.
The predicates that will be allowed by type checking will have type 
$[int,int]$ and the other predicates will have type $[bottom,bottom]$.
We restrict the number of clauses to three.

\paragraph{Result}

Figure \ref{fig:space-ratio} has a plot showing how the time necessary for
traversing the search space is influenced by the ratio of type matching versus
all predicates. The second plot shows the same behaviour but for the number
of proof steps by the untyped and simply typed systems.

There is enough evidence to \emph{reject} the hypothesis that the ratio of correctly
typed background predicates \emph{does not} matter. Clearly the ratio has significant
influence on the time taken and the number of decisions made by the algorithm.
Note that the untyped system will not care about the ratio of the matching
typed predicates as it will not consider this feature at all, hence leading
to its constant behaviour.

The plots are again best explained by degree four polynomials, again
demonstrating the need for better fitting theory. The time plot is interesting
in particular for that it reveals that the \metast{} with type checking turned
off, which is how this test was conducted for the untyped results, apparently
has overhead from the types that are carried around, even more so than the
typed system, as clearly demonstrated by the case when the ratio is $1.0$.
This is likely due to an increased cost for making placeholder variables
for the types that are not being checked, which appears to have non-trivial 
cost in SWI-Prolog.

\begin{figure}[H]
\begin{subfigure}{\textwidth}
\centering
\includegraphics[width=0.99\textwidth]{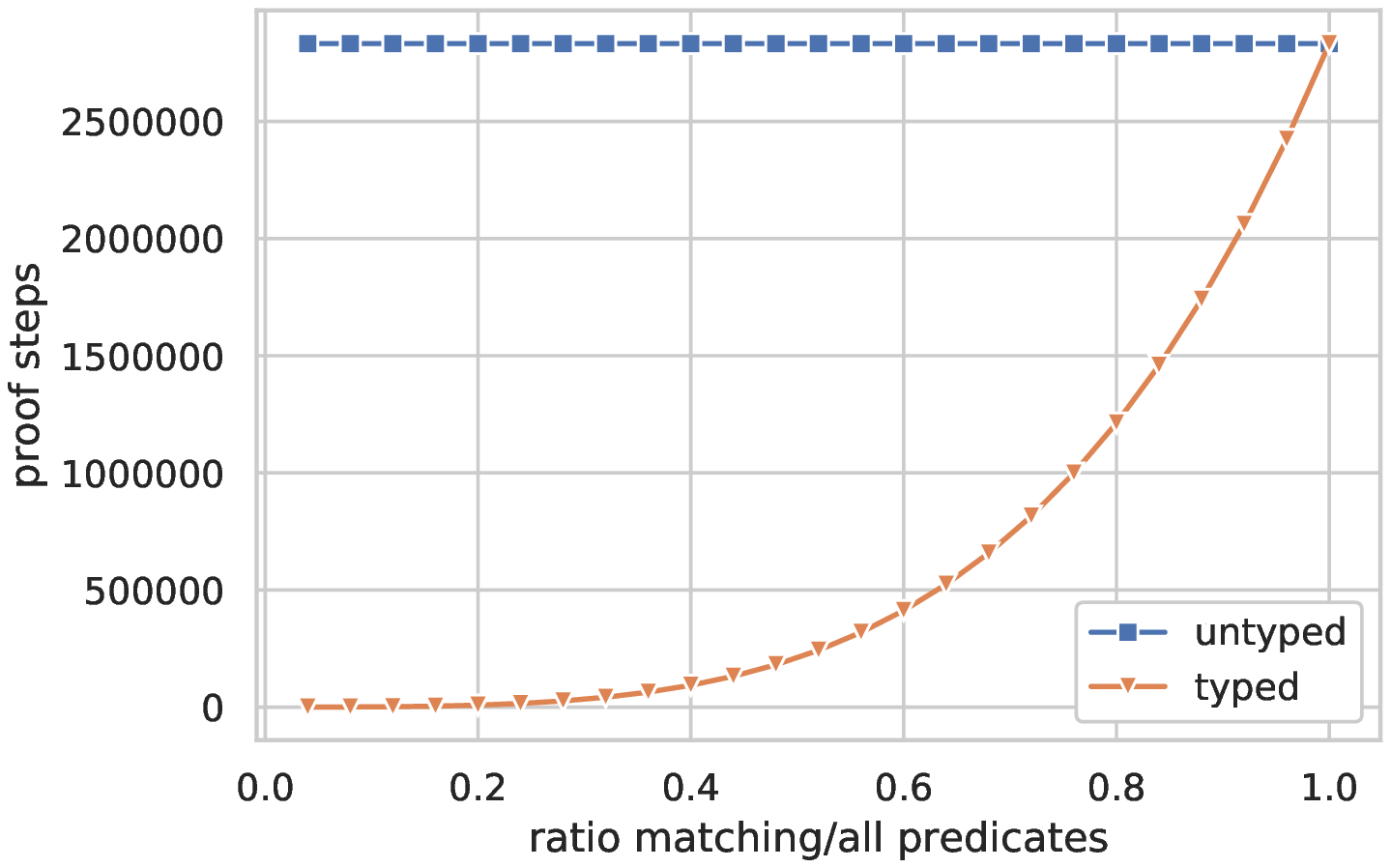}
\end{subfigure}

\hspace{3.2ex}
\begin{subfigure}{\textwidth}
\centering
\includegraphics[width=0.89\textwidth]{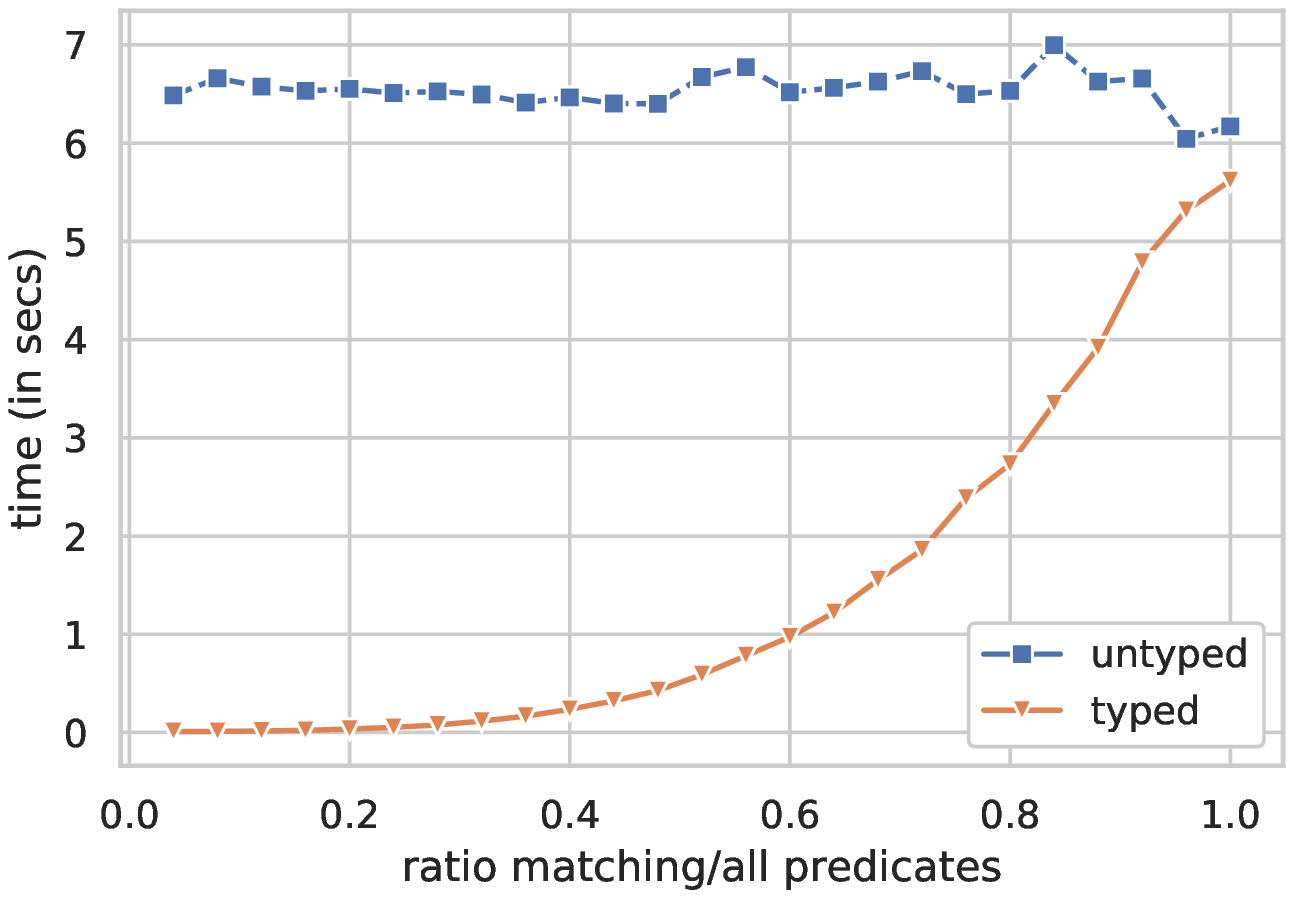}
\end{subfigure}
\caption{Number of proof steps and time for increasing ratio of type matching 
predicates out of all predicates.}
\label{fig:space-ratio}
\end{figure}

\stepcounter{expcount}
\subsection{Experiment \arabic{expcount}: Simply Typed Droplasts}
\label{sec:simple-droplasts}

As a final experiment we take a popular exercise from the literature \citep{kitzelmann2007data},
namely the \texttt{droplasts} program. This polymorphic program takes a list of lists
and drops the last element from each of the inner lists. The program learned is:

\begin{center}
\begin{tabular}{c}
\begin{lstlisting}
droplasts(A,B):-map(A,B,droplasts_1).
droplasts_1(A,B):-reverse(A,C),droplasts_2(C,B).
droplasts_2(A,B):-tail(A,C),reverse(C,B).
\end{lstlisting}
\end{tabular}
\end{center}

\begin{hypothesis}
\label{hypo:simple}
\metast{} is not able to improve on the untyped system, in the time required
and the number of proof steps needed, in a (semi-)realistic setting.
\end{hypothesis}

\paragraph{Setup}
This experiment is conducted stochastically. We generate small
random input examples (outer and inner lists (of integers) of length between 
2 and 5)\footnote{The rather limited lengths are chosen so that we can reuse
most of this experiment when we consider the refinement experiment for the following chapter}
and run them through a reference implementation to obtain correct
positive examples. As we are interested in the effect of type checking on
the search space based on the background predicates, we fix the number of 
(positive) examples generated to three.

We provide the synthesis system with predicates for a list \texttt{concat}
relation (appending elements at the back), 
the \texttt{tail} relation, the \texttt{reverse} relation 
and the two-place identity relation, all with appropriate polymorphic types.
The type of the examples is set to $[list(list(int)),list(list(int))]$.
We give it the \texttt{chain} metarule, along with metarules
for abstraction to the following predicates. The higher-order predicates
made available are the \texttt{map}, the \texttt{reduceback} and the \texttt{filter}
relations.

For the experiment we add additional typed predicates, which for sake of
execution time we take to be simple, forgoing excessive costs associated
with non-determinism and resolution\footnote{The typed system would for the
most part not incur these costs, hence adding an additional cost for resolution
and non-determinism would mostly be a tool to manipulate the results in
your favor.}. The additional predicates only match on particular values for
their arguments, where the values are randomly chosen from a small distribution.
The types of predicates are correct for the values
of the argument and are either: $bool$, $nat$, $int$, $list(int)$, 
$list(list(int))$, or $list(list(X))$.

\begin{figure}[H]
\begin{subfigure}{\textwidth}
\centering
\includegraphics[width=0.99\textwidth]{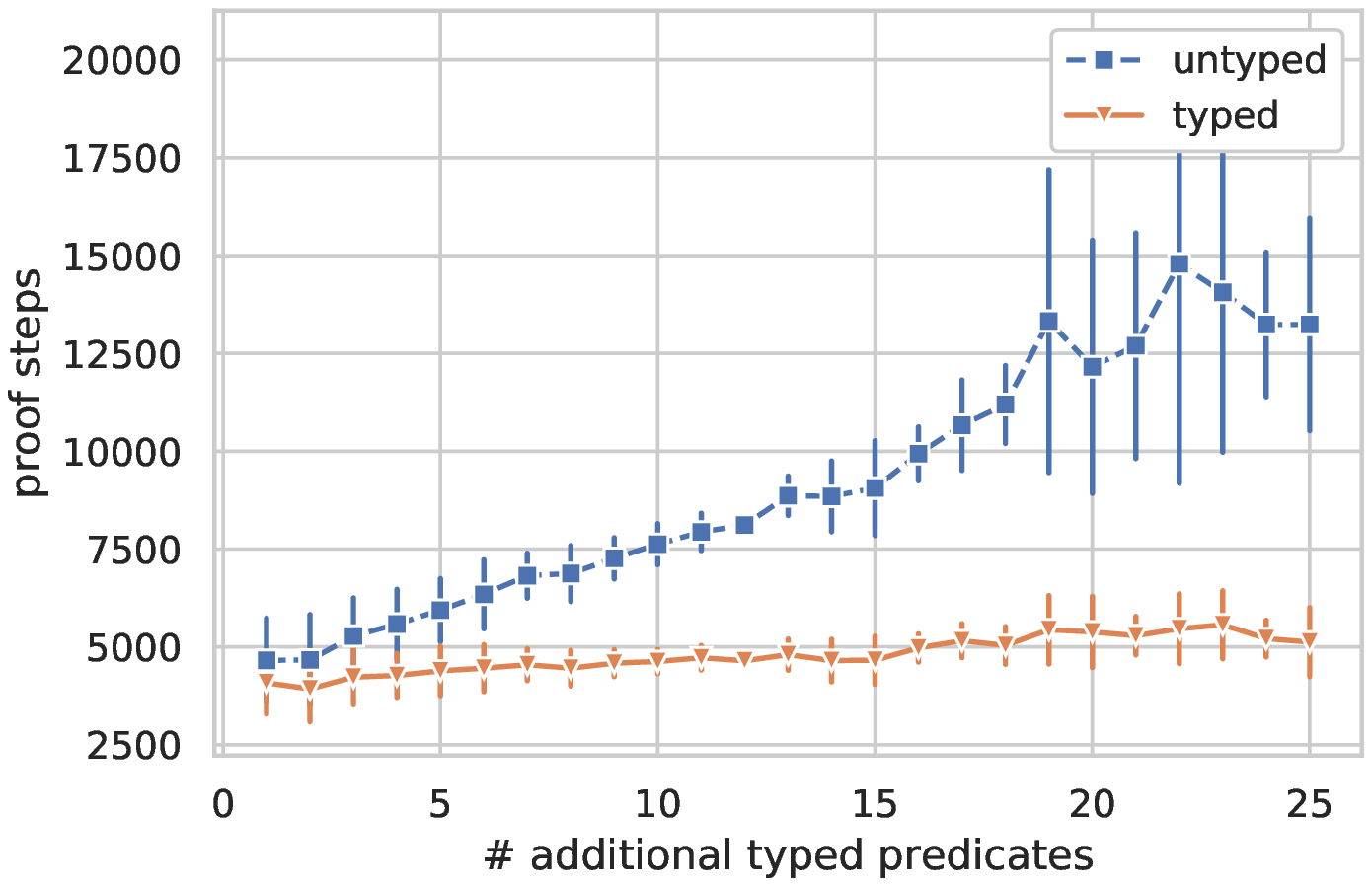}
\end{subfigure}

\begin{subfigure}{\textwidth}
\centering
\includegraphics[width=0.97\textwidth]{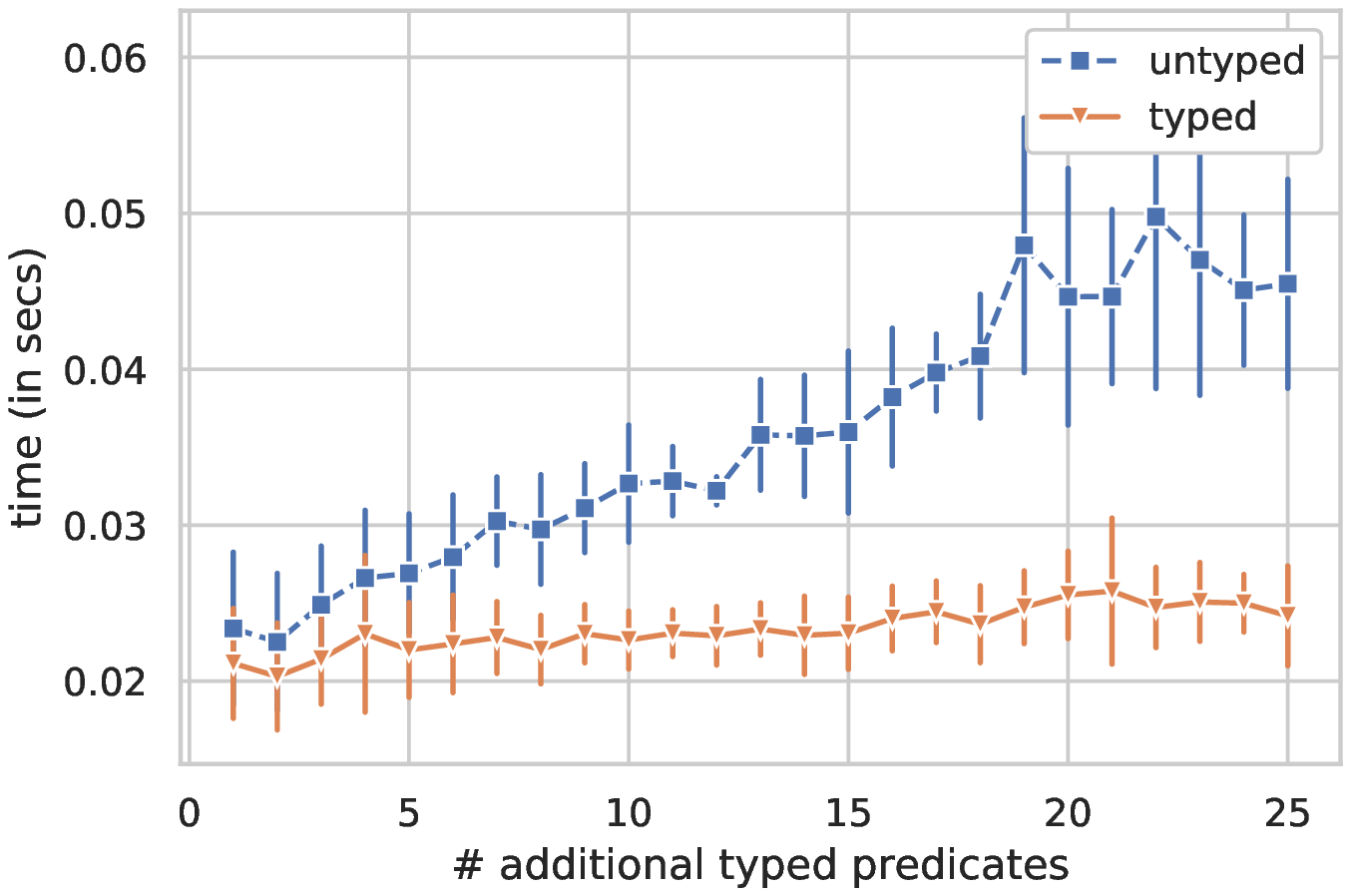}
\end{subfigure}
\caption{Average number of proof steps and time for increasing number of typed background predicates.
Standard deviation is depicted by bars.}
\label{fig:simple-droplast}
\end{figure}

For each number of additional background predicates we run ten trials,
for each trial generating random positive examples and predicates.
We average over the trials and calculate the standard deviation of the sample.

\paragraph{Result}

The plots in figure \ref{fig:simple-droplast} depict the average time and number of
proof steps required by the untyped and typed systems. The standard deviations
are included as bars.

These graphs are reasonable evidence for concluding that the typed system is
able to outperform the untyped system in a non-toy synthesis example,
thereby rejecting hypothesis \ref{hypo:simple}.

\paragraph{Analysis}
The advantage of type checking observed in the previous, deterministic experiments
is less pronounced in this experiment. The main reason for this is that while
in the previous tests the untyped program needed to explore the entire 
space of possible programs, 
i.e.~with each background predicate at each possible position
in all clauses, now the untyped system is able to discard parts of its search
space based on the \emph{values} in the background predicates not unifying.

\chapter{Synthesis with Refinement Types}

Based on the success of polymorphic type checking in making synthesis more efficient,
it is natural to consider whether more expressive types can be leveraged for
further improvement. The extension to polymorphic types we will consider in this
chapter is refinement types, i.e.~polymorphic types with an additional proposition
restricting inhabitants of the type (see section \ref{sec:typed-type-def}
for a definition). We will often use the term \emph{simple types} to refer
to polymorphic types without refinements.

First we discuss how the user specifies refinement types for the
background knowledge.
As the MIL algorithm conceptually only needs minor adjustments, we next 
introduce the adapted algorithm \metart{}.
We show how a refinement proposition representing the entire program 
can be obtained from the proof structure maintained by the algorithm.

Refinement checking is accomplished by proving (un)satisfiability by a 
SMT solver.
We discuss the levels of refinement expressivity available, 
and how there is a tradeoff to be made.
The chapter closes with theory and experimental results.

\section{Representation of Refinements}

The reader is referred to section \ref{sec:typed-type-def}
for the definition and some of the syntax used for refinement types.
We assume familiarity with the problem statement for program synthesis
using polymorphic typed predicates, as well as with the way users provide their 
background knowledge to the \metast{} algorithm, as can be found
in section \ref{sec:simple-problem}.

The user supplied positive and negative examples are unchanged versus
the simply typed case: the user provides a single consistent (non-refinement)
type for all the examples.
As in the case of polymorphic type checking,
we need the user to supply refinement types for the background knowledge. 
The refinement types on the predicates should be such that they are consistent
with the predicates, i.e.~for any predicate 
$p(A_0,\ldots,A_k):[T_0,\ldots,T_k]\<\phi\>$, which gets
its arguments instantiated to $a_0,\ldots,a_k$, it is never the case
that $p(a_0,\ldots,a_k)$ holds while $\phi[a_0/A_0,\ldots,a_k/A_k]$ is
false. Note that the user is not otherwise restricted, they can choose
to use the most general refinement ($true$) or to make it as precise as they
like (even as far as the refinement completely characterizing the predicate).

For specifying refinements in Prolog we require a way to refer to the arguments.
Instead of just annotating the predicate's name with a type we also write out
its formal arguments when asserting that a predicate belongs to the background
knowledge.
The primitive predicate assertions gain an additional argument for the refinement on 
the simple type. The simple type itself is appended to the atom representing
the predicate's name and formal arguments. For example 
\pc{prim(tail(A,B):[list(X),list(X)],$\<length(A) = length(B)+1\>$)}, is
the predicate representing the tail relation with the refinement stating
the length property that always holds of its arguments. Note that the
refinement bracket syntax ($\<\ldots\>$) made its way into our Prolog notation.
For now we hide the actual syntax used in writing down refinements by this 
notation and comeback to the refinement specification language in 
section \ref{sec:ref-expr}.

The interpreted predicates gain an additional assertion for the specification
of the refinement type. We will need to keep track 
of structure of the program derived, 
and hence will keep track of the existentially quantified predicate
names in the interpreted clause.
As an example, the assertion of the map predicate being an interpreted
background predicate becomes:
\begin{center}
\begin{tabular}{c}
\begin{lstlisting}
interpreted([F],map([],[],F):[list(X),list(Y),[X,Y]] :- []).
interpreted([F],map([A|S],[B|T],F):[list(X),list(Y),[X,Y]] :- 
             [F(A,B):[X,Y],map(S,T,F):[list(X),list(Y),[X,Y]]]).
interpreted_ref(map(C,D,F),(*$\<length(C)=length(D)\>$*))
\end{lstlisting}
\end{tabular}
\end{center}

The metarules will not need to keep track of refinements (directly)
and hence remain unchanged versus the simple type problem statement,
i.e.~all atoms are decorated with simple types.
The definition of consistent hypotheses,
and that of MIL learner are essentially unchanged: the programs learned
are still simply typed, and the only adjustment needed is that the supplied
background predicates are refinement typed.

\section{From a Program Derivation to a Refinement}

This section starts out with observing how refinements on predicates
give rise to a refinement over a program.
Next we present how refinement type checking can be integrated
at a high level into the MIL algorithm. This high-level algorithm will
delegate the type checking to a subroutine, hiding the complexity.
We subsequently look how this type checking subroutine is able to 
construct a single proposition that needs to be checked for satisfiability.

\subsection{Backward Action on Refinements}
\label{sec:ref-backward}

We now look at how the validity of substituting
the head of a definite clause with its body gives rise to inferring refinements.
Suppose we are in a position of needing to prove that an atom $Q(X,Y)$ is
an inhabitant of its type $[T_X,T_Y]\<\phi_Q\>$, where $\phi_Q$ might be 
a unknown refinement and $X$ and $Y$ might both be variables 
(and not yet concrete terms) of as of yet undetermined type.

\paragraph{Substitution of refinements}
Suppose that after applying the chain metarule that $Q$ is assigned the 
body $Q(X, Y) \ot R(X,A),S(A,Y)$
with the following typing (again involving unknowns):
\begin{align*}
    R(X,A) &: [T_X,T_A]\< \phi_R \>\\
    S(A,B) &: [T_A,T_Y]\< \phi_S \>\\
\end{align*}

%We know that we can conclude that $Q(X,Y) : [T_X,T_Y]\<\phi_Q\>$ if
%$R(X,A) : [T_X,T_A]\<\phi_R\>$ and $S(A,Y) : [T_A,T_Y]\<\phi_S\>$.
If the refinement of $Q(X,Y)$ was unknown we now know that $\phi_R \wedge \phi_S$
is an accurate substitution for $\phi_Q$ as the just invented definite clause 
would be the only way of deriving $Q(X,Y)$.
To see that this is an accurate substitution observe 
that any occurrence of $Q(X,Y)$ in the program may be replaced by the newly
constructed body. If $\phi_Q$ is not unknown it is the case that there
is already another body assigned to $Q(X,Y)$, which means
we are adding a disjunctive clause. The refinement $\phi_Q$ needs to
be updated to $\phi_Q' = \phi_Q \vee (\phi_R \wedge \phi_S)$.

\paragraph{Backward action}
From the above description it follow that the types in the body of an invention 
have a kind of 
\emph{backward action} with regard to the type of the predicate that is being 
invented. 
This backward action leads to additional named (argument) variables occurring in 
refinement types,
names not present in a predicate head's arguments. Therefore
a context needs to be maintained for any such existentially qualified variables.
The significance of this backward action is that is leads to a \emph{grand 
refinement} for the entire program, i.e.~a single proposition whose satisfiability
determines whether the program is still consistent. For example, it could 
be that $R(X,A)$ has to be invented leading to an adjustment of its refinement type, which
would in turn change $Q(X,Y)$'s type. In the same way $Q(X,Y)$ could occur
in the body of the clause invented for predicate of the examples.
Hence we have that the entire program's refinement type is influenced
by any adjustment to a type due to filling in a definition for a
predicate.

Note that for interpreted predicates we have that the same property holds:
the refinement for the head of the interpreted clause can be made more accurate
by conjuncting it with the refinements of the predicates that occur in the 
clause's body.

\subsection{High-level algorithm: \metart{}}

Figure \ref{fig:metart} contains the code for \metart{}, the refinement type
checking Meta-Interpretive Learning algorithm. The algorithm is, in essence,
the \metast{} algorithm from the previous chapter with additional pruning.
The code in bold notes all the changes made to implement refinement checking. 
The same four disjunctive clauses remain and fulfill the same tasks. 

\begin{figure}
\begin{center}
\begin{tabular}{c}
\begin{lstlisting}
learn(Pos,Neg,Type,Prog):-
  map(decorate_types(Type),Pos,PosTyped),
  map(decorate_types(Type),Neg,NegTyped),
  prove(PosTyped,[],Prog),
  not(prove(NegTyped,Prog,Prog)).
prove([],Prog,Prog).
prove([Atom|Atoms],Prog1,Prog2):-
  prove_aux(Atom,Prog1,Prog3),
  prove(Atoms,Prog3,Prog2).
prove_aux(Atom:DT:GT,Prog,Prog):-
  prim(Atom:DT),!,
  prim(Atom:GT),
  (*\texttt{\textbf{check\_refinement(Prog),}}*)
  call(Atom).
prove_aux(Atom:DT:GT,Prog1,(*\pc{\textbf{Prog3}}*)):-
  interpreted((*\texttt{\textbf{Subs,}}*)Atom:DT:-BodyDT),
  interpreted((*\texttt{\textbf{Subs,}}*)Atom:GT:-BodyGT),
  combine_types(BodyDT,BodyGT,Body),
  (*\pc{\textbf{Prog2=[inter(Atom,DT,GT,Subs)|Prog1]}}*)
  (*\texttt{\textbf{check\_refinement(Prog2),}}*)
  prove(Body,(*\pc{\textbf{Prog2,Prog3}}*)).
prove_aux(Atom:DT:GT,Prog1,(*\pc{\textbf{Prog3}}*)):-
  member(sub(Name,GTinv,Subs),Prog1), 
  check_unifies_with(GTinv,DT),
  GT=GTinv,
  metarule(Name,Subs,(Atom:DT:-BodyDT)),
  metarule(Name,Subs,(Atom:GT:-BodyGT)),
  combine_types(BodyDT,BodyGT,Body),
  (*\pc{\textbf{Prog2=[inv(Name,Atom,DT,GT,Subs)|Prog1],}}*)
  (*\texttt{\textbf{check\_refinement(Prog2),}}*)
  prove(Body,(*\pc{\textbf{Prog2,Prog3}}*)).
prove_aux(Atom:DT:GT,Prog1,(*\pc{\textbf{Prog3}}*)):-
  metarule(Name,Subs,(Atom:DT:-BodyDT)),
  metarule(Name,Subs,(Atom:GT:-BodyGT)),
  combine_types(BodyDT,BodyGT,Body),
  (*\pc{\textbf{Prog2=[sub(Name,Atom,DT,GT,Subs)|Prog1],}}*)
  prove(Body,(*\pc{\textbf{Prog2,Prog3}}*)).
\end{lstlisting}
\end{tabular}
\end{center}
\caption{\metart{}: refinement type checking MIL.}
\label{fig:metart}
\end{figure}

Each disjunct maintains in its \pc{Prog} arguments what the currently derived
program is. For the first disjunct this is done by unifying the selected
predicate name with one of the variables already in \pc{Prog}, hence the change
in the program \pc{Prog} needs no further work.
The other disjuncts explicitly keep track of the structure of the proof constructed
by the algorithm. For invented clauses a meta-substitution is already
kept track of. Uses of higher-order background predicates and of 
invented clauses are now also saved.

The \pc{check\_refinement} predicate achieves pruning by deriving the
proposition representing the constraints imposed on the entire program
by the refinements. The stored proof of the program gives rise to the 
grand refinement, the proposition whose satisfiability determines whether
the program being considered is still a viable option, or else is already
inconsistent. When the grand refinement is satisfiable the call to
\pc{check\_refinement} holds for the supplied program and the search continues.
If the refinement of the supplied program is proven unsatisfiable, 
the proof search procedure of Prolog starts backtracking, discarding 
(at least) the last choice made for the program.

When considering just filling in the existential variables from a clause body
we have that an inconsistent refinement will not become consistent upon 
filling in more variables in the clause. This monotonicity property carries 
over to the grand refinement of the entire program, i.e.~once proving
unsatisfiable subsequent additions to the program cannot make the
grand refinement satisfiable.
This observation is enough to guarantee sound pruning of the search space.

\subsection{Tree-Shaped Grand Refinement}

As noted in the previous section the algorithm maintains the structure of
the program constructed in the form of how the invented and interpreted clauses
are used. This structural information encodes where predicates occur in the 
program and over which variables they operate.

\paragraph{Tree-shape derivation}

The Meta-Interpretive Learning approach to synthesis extends 
Prolog's backward-chaining algorithm with additional methods for proving atoms. 
The backward-chaining algorithm come downs to the idea that a goal 
atom is proven by unifying the goal with the head of a definite clause 
leading to the body atoms of this clause becoming the goals.
Hence a goal atom has child goal atoms, which in turn have child goal atoms, etc.,
until a goal atom is an asserted fact (i.e.~has not body to prove). This means
that the proof of a goal atom forms a tree of goals.

The derivation maintained in the \metart{} algorithm maintains this structure
in \pc{Prog},
but just for the proof steps that involve interpreted and invented clauses.
For usages of interpreted predicates we store \pc{inter(Atom,DT, GT,Subs)},
where \pc{Atom}'s predicate identifies the interpreted clause used.
The \pc{inter(\ldots)} atom encodes the information needed to reconstruct goal 
body atoms, instantiating known predicate variables with the substitution \pc{Subs}.
For invented clauses we have that 
\pc{sub(Name,\ldots)} and \pc{inv(Name,\ldots)} use \pc{Name} to identify the 
metarule used, thereby giving access to the body goals. 
The leafs of this proof tree are the atoms that are either proved by a primitive,
meaning that the chosen primitive will be stored in the \pc{Sub} substitution
of the parent goal, or are atoms who are yet to be proven (whose predicate
symbol might be a variable).
Hence the information stored in \pc{Prog} is sufficient to recreate the proof 
tree with known predicate variables resolved and leaving unknown predicates
as variables.

\paragraph{Constructing the Grand Refinement}
As follows from section \ref{sec:ref-backward}, the grand refinement can be
directly derived from this proof structure. We explain a traversal of the
derivation tree whereby the grand refinement is constructed and at the same
time a \emph{variable context} is maintained. The variable context is the
set of typed variables that occur in the derivation,
along with the values that they are assigned to (in case an argument variable
has already been assigned a value).

\begin{figure}
\begin{center}
\begin{tabular}{c}
\begin{lstlisting}
atom_to_refinement(Pred(Args):Ty,Prog,[],(*$\<true\>$*)):-
  var(Pred).
atom_to_refinement(Pred(Args):Type,Prog,Context,Refinement):-
  prim(Atom:Type,Refinement),
  typed_args_to_context(Args,Type,Context).
atom_to_refinement(Pred(Args):Type,Prog,Context,Refinement):-
  member(inter(Pred(Args),Type,GT,Subs),Prog),
  interpreted(Subs,Pred(Args):Type:-Body),
  interpreted_ref(Pred(_),InterRef),
  body_to_refinement(Body,Prog,BodyContext,BodyRef),
  typed_args_to_context(Args,Type,InterCtx),
  append(BodyContext,InterCtx,Context),
  Refinement=and(BodyRef,InterRef).
atom_to_refinement(Pred(Args):Type,Prog,Context,Refinement):-
  (member(app(Name,Pred(Args),Type,GT,Subs),Prog);
  (member(sub(Name,Pred(Args),Type,GT,Subs),Prog)),
  metarule(Name,Subs,(Pred(Args):Type:-Body)),
  body_to_refinement(Body,Prog,BodyContext,Refinement),
  typed_args_to_context(Args,Type,InterCtx),
  append(BodyContext,InterCtx,Context).

body_to_refinement([],Prog,Context,(*$\<true\>$*)).
body_to_refinement([Atom|Atoms],Prog,Context,Refinement):-
  atom_to_refinement(Atom,Prog,Ctx1,Ref1),
  body_to_refinement(Atoms,Prog,Ctx2,Ref2),
  append(Ctx1,Ctx2,Context),
  Refinement=and(Ref1,Ref2).
\end{lstlisting}
\end{tabular}
\end{center}
\caption{Prolog code for converting a derivation to a Grand Refinement.}
\label{fig:grand-ref}
\end{figure}

The details of the algorithm for converting a derivation to its 
grand refinement are in figure \ref{fig:grand-ref}. The basics are
that two mutually recursive clauses build up the grand refinement bottom up.
Leafs, i.e.~those atoms whose predicate is a variable or a primitive, 
are directly convertible, see the first two disjuncts of 
\pc{atom\_to\_refinement}.
In the case of interpreted predicates the body refinement is first
derived, and this refinement is conjuncted with the refinement that was supplied for the interpreted clause.
For every atom the arguments are added to the context.
In \pc{body\_to\_refinement} a list of atoms is a clause body, which represents
a conjunction of atoms, hence the main task of this clause is to collect
contexts and conjunct refinements.

The grand refinement is derived by supplying \pc{atom\_to\_refinement}
with the example goal that algorithm is currently trying to prove, 
along with the current program derivation. 
As a result we obtain a variable context, containing variable names with
typing (and possibly values), and a single large proposition: the grand refinement. 

\section{Grand Refinement Checking}
\label{sec:ref-checking}

The only issue not dealt with in the previous sections is that of 
establishing whether the grand refinement is satisfiable. Satisfiability of
the refinement is defined as there being an assignment of the variables
in the context such that the grand refinement is true.

This section explores the possibilities that Satisfiability Modulo Theories (SMT)
provide as a framework for specifying refinements. In this framework
the grand refinement corresponds to a SMT logic formula.
SMT solvers will be used to try to 
prove unsatisfiability (also called inconsistency) of this formula.

\subsection{SMT Solvers}

The main motivation for choosing to explore the effectiveness of
refinement types in pruning the search space is that SMT solvers have proven
to be very efficient in solving satisfiability problems.

Satisfiability Modulo Theories (SMT) are theories for stating logic problems 
regarding a set of formulas having a satisfying model. The formulas are expressed in
a suitably restricted logic. The satisfiability problem is encoded in a language
supported by SMT solvers.
A common standard, with widespread support, is the SMTLIB (2.0) language \citep{barrett2010smt}. 
SMT solvers are special purpose programs utilizing the best available
algorithms to prove satisfiability, often relying on heuristics.
Forerunners in performance and support of new logics are 
the Z3 solver \citep{de2008z3}, and the CVC4 solver \citep{barrett2011cvc4}.

\subsection{Language of Refinements}

Up till this point we have relied on mathematical syntax for specifying refinements.
The only requirement we have had on refinements is that they are propositions
that may mention the arguments of a predicate, where we left the availability
of certain functions and notations out of scope.

With our choice for solving satisfiability of the grand refinement fixed,
we can use this decision to guide our specification of refinements.
As any language we choose needs to be translated to a format that the SMT solvers
are able to work with, the simplest solution is to take the syntax of 
SMTLIB.

\paragraph{Variables in refinements}
The main feature on top of SMTLIB that we need is to be able to correctly
keep track of the named variables in the refinements. The chosen solution
is to move from a refinement being a single string to list of strings and
terms. An example for such a refinement is for the map predicate:
\begin{center}
\begin{tabular}{c}
\begin{lstlisting}
interpreted_ref(map(A,B,F):[list(X),list(Y),[X,Y]],
                  ['(= (length ',A,') (length ',B,'))']).
\end{lstlisting}
\end{tabular}
\end{center}

The first thing to notice is the Lisp-like syntax of the refinement, and that
the SMTLIB language has support for functions, e.g.~\pc{length}.
The significance of the interspersing of variable names is that a variable
occurring multiple times in the program derivation (and hence also in the 
grand refinement), are identified with one another.
This also makes it easy to translate such a single derivation variable
to just one SMT variable. The SMT translation of the grand
refinement just has this single variable name in place of all the occurrences
of the variable.

\paragraph{Encoded problem}

The translation of the grand refinement to a SMT problem now proceeds as follows.
First the context variables are considered, which are the same logical
variables that occur in the refinement. For each variable there is 
a declaration of a new SMT variable with a new name, with typing as it occurs 
in the derivation.
If the variable has a known value in the derivation a SMT equality assertion
is generated for the variable with this value.
Note that there may
occur higher-order variables in predicate arguments, 
but these will never occur as variables in refinements and hence are filtered out.

For the grand refinement we have that each single refinement occurring in it
can now be ``folded
down'' from a list of strings and Prolog variables to a single string.
The names chosen for the SMT variables are used as substitutions for the 
variables that occur in the refinements, where upon string concatenation becomes
available.
The structure of the grand refinement itself, involving disjunctions
and conjunctions over refinements, can now be made
into a single string by replacing nodes such as \pc{and(ref1,ref2)} by 
\pc{RefStr} from the following code:

\begin{center}
\begin{tabular}{c}
\begin{lstlisting}
ref_to_SMT(ref1,smt1),ref_to_SMT(ref2,smt2),
        str_concat(["(and ",smt1," ",smt2,")"],RefStr).
\end{lstlisting}
\end{tabular}
\end{center}
The set of formulas given to the SMT solver are thus: the variable declarations,
assertions for values of variables (as far as they are known),
and the translation of the grand refinement type.

\section{Expressiveness of Refinements}
\label{sec:ref-expr}

Having chosen the syntax of the type refinements, this section looks
at the available choices in regards the logic theories that SMT solvers support.
As has been the theme of this document we identified higher-order predicates
as promising, with as argument that many standard list functions/predicates
have simple refinements. We therefore focus on SMT logic theories that are
able to reason over lists.

\subsection{Z3 Sequence Theory}

The first candidate is the Z3 Sequence \citep{SMTseq} theory. This theory
is able to reason over lists of bounded length and includes such function symbols
as: \pc{seq.concat}, \pc{seq.len}, \pc{seq.indexof}, \pc{seq.extract}, etc.
The theory is already undecidable, but it is the smallest theory that we
identified that includes list reasoning. 

In this theory it is very easy to write refinements using just the set of
available function symbols. As our main example, the map predicate's 
refinement assertion looks very familiar:
\begin{center}
\begin{tabular}{c}
\begin{lstlisting}
interpreted_ref(map(A,B,F),
                  ['(= (seq.len ',A,') (seq.len ',B,'))']).
\end{lstlisting}
\end{tabular}
\end{center}

Beyond its undecidability, the main issue with this theory is that it lacks
expressiveness. For example, we might want a refinement on a sort predicate
that states that all the elements of the input and the output lists are the same.
Such a refinement is not possible in this theory.

\subsection{DTLIA: Quantifiers, Datatypes and Arithmetic}

Lists are the prime example of algebraic datatypes (ADTs). In the ADT formulation
of lists there is only the empty list constructor and the \pc{cons} list
constructor for an element and a second list. The only operations available are
to pattern match on these two constructors. Embracing algebraic datatypes
leads more possibilities beyond just lists, e.q.~option types and record types.
The ADT formulation of lists comes with no functions pre-defined, that is,
notions such as length will have to be user-defined. Many important functions
on lists are recursive, e.g.~the length function. When reasoning over lengths
it is often useful to also compare lengths. This means we also need some basic
arithmetic notions. 
%To conclude: we need ADT support, we need recursive function definition support
%and we need arithmetic support.

Recently there has been work on SMT solvers to support algebraic datatypes, 
the logic fragment of which is called $DT$ \citep{reynolds2015decision}. 
This theory on its own is decidable, though with quantifiers it is not.
In other recent work progress\footnote{The work is so recent that I discovered a unsoundness issue in the implementation. The issue was promptly fixed: https://github.com/CVC4/CVC4/issues/2133} has been made on supporting recursive functions on these datatypes \citep{reynolds2016model}. 
SMT solvers translate recursive functions definitions (on ADTs)
to quantified formula, therefore we have to relinquish decidability.
The need to support simple arithmetic is satisfied by the Linear Integer 
Arithmetic (LIA) theory.
The combination of theories $DT$ (with quantifiers over Uninterpreted Functions)
and $LIA$ is the theory $(UF)DTLIA$. Only the CVC4 solver implements this theory.

The complexity of the recursive function definitions can be hidden from
the user for standard definitions, e.g.~\pc{dt\_length} in the following:
\begin{center}
\begin{tabular}{c}
\begin{lstlisting}
interpreted_ref(map(A,B,F),
                  ['(= ',dt_length(A),' ',dt_length(B),')']).
\end{lstlisting}
\end{tabular}
\end{center}

Behind the scenes the \pc{dt\_length} predicate is able to insert the appropriate
SMTLIB function definition in the set of formulas handed to the SMT solver.
The below code shows such a definition. Note that SMTLIB does not support
the polymorphism of this example (the example uses \pc{Y} as a type parameter), 
but that the generated code must actually generate a separate function definition
based on the types in each instance of the function being in the grand refinement.

\begin{center}
\begin{tabular}{c}
\begin{lstlisting}
(define-fun-rec dt_length ((x (List Y))) Y
 (if-then-else (= x (as nil (List Y)))
      0
      (+ 1 (dt_length (tail x)))
  )
 )
\end{lstlisting}
\end{tabular}
\end{center}

For user-defined recursive SMT functions the effort required to integrate with
the system becomes significant.
The implementation of this system is quite non-trivial while the details are 
rather uninteresting. The document will not look further into the matter.

The $DTLIA$ is our preferred theory in terms of expressiveness. Recursive functions
give enough power to express subset inclusions on lists, hence it becomes
possible to express the property of the \pc{sort} predicate that the elements
are permuted.
It is interesting to note that the recursive functions are so powerful that they
essentially allow us to fully encode (first-order) Prolog predicates in
the SMT logic. Hence users get the option to choose, on a sliding scale, between
very accurate refinements and very coarse refinements.
The understanding is that these more complicated refinements might take more
time to reason over, though could also be useful in detecting inconsistency earlier.

%\todo{Using Z3's builtin Sequence sort is not feasible for getting beyond
%the restricting given definitions, e.g. a definition of subsequence is already
%not decidable by Z3, even for very simple lists.}

%\todo{The merit of invoking z3 on filling an existing invention (fully filled in)
%is questionable as it is true that Prolog will do fewer calls to prove\_aux, 
%but to create the refinement the entire derivation is still traversed, something
%which would otherwise correspond to the prove\_aux calls. Note that it might
%still be the case that there are gains when a non-fully filled in invention
%is re-used and discarded or when the concrete execution of the predicates 
%used are computationally expensive, while their refinement representation isn't.
%}

%\todo{Discuss why solvers underperformed (e.g. Z3 does not have a precise logic
%    for these instances), and how solvers can be instructed to do better, e.g.
%CVC4's options for finite model finding and inductive proofs.}

\section{Theoretical Results}

The results from the previous chapter can be directly lifted to the setting
of refinement types. The reason for why the propositions can be directly
restated and their proofs only slightly modified is that the refinement
types have the same soundness properties for the synthesis algorithm as
simple types. In particular the main feature is the inconsistency
pruning of the search space.

\begin{proposition}
The programs found by the \metaai{} algorithm (which can be assigned
a polymorphic types with refinements)
and the \metart{} algorithm are exactly the same.
\end{proposition}

The proof follows again by that a depth-first search algorithm is used
to traverse the program hypothesis space and that programs are still encountered
in the same order, just that the some inconsistent programs have been skipped
over by the \metart{} algorithm due to type checking. For a more precise argument see section \ref{sec:simple_theory}.

The reduction of the space of $\Htwo$ programs also directly applicable to 
the refinement type system. The refinements are for the most part an improvement with regard to
the granularity of the types, i.e.~they will further restrict the possibilities
for choices of predicates in a clause. This means that the result regarding
the \emph{worst case proportional constant} applies but that refinements in
the background knowledge will shrink this worst case ratio further.

\section{Experimental Results}

We present experiments which check whether there is sufficient evidence to 
claim that the current implementation of
refinement type checking has significant advantages, i.e.~whether we are justified
in \emph{rejecting} the following:

\begin{hypothesis}
\label{hypo:ref-impr}
Refinement type checking cannot reduce the number of proof steps 
compared to non-refinement polymorphic type checking.
\end{hypothesis}

As in the previous chapter we focus on the \texttt{droplasts} program. This
program takes a list of lists and drops the last element from each of the inner lists.
We perform statistical experiments to evaluate whether refinement type checking has
any benefits versus only simple types checking (and by extension versus the
untyped system).

In the first experiment we add predicates with polymorphic types that compose well,
but whose refinement types lead to being able to decide that some combinations
of predicates in a clause will not work. This test is primarily to show a difference
between simple type checking and refinement type checking.
The second experiment tries to be a bit more general and includes predicates
that are less sensible, though allow for a bigger difference between
the untyped case and the typed cases.

\stepcounter{expcount}
\subsection{Experiment \arabic{expcount}: Droplasts with Sensible Predicates}

The program we are synthesizing is \texttt{droplasts}. We perform a statistical
experiment where we have random background knowledge 
containing predicates with refinement types and are iteratively 
adding additional refinement typed predicates to 
check how the refinement type checking system behaves when the composition of
refinements rules out certain programs.

\paragraph{Setup}

We follow the setup of the simply typed droplasts experiment of section \ref{sec:simple-droplasts}.
We generate small random input examples, outer and inner lists (of integers) of length between 
2 and 5, and run them through a reference implementation to obtain correct
positive examples.
We choose this length restriction for the  examples, because the
SMT solver needs more time for larger lists. 
Adding additional examples would make the found programs more accurate,
but would come at the cost of even longer synthesis times (due to the SMT solver
being invoked more often).
As we are interested in the effect of type checking on
the search space we fix the number of (positive) examples generated to three.

We provide the synthesis system with the following predicates:
\begin{itemize}
    \item \texttt{concat:[A:list(T),B:T,C:list(T)]$\<length(A) + 1 = length(C)\>$}
        \subitem The relation that appends an element at the back, where
        \texttt{length} is the length function over ADT lists definable in
    SMT-LIB.
    \item \texttt{tail:[A:list(X),B:list(X)]$\<A=cons(\_,B)\>$}
        \subitem The relation taking of the head of a list, with accurate
        refinement representation.
    \item \texttt{reverse:[A:list(X),B:list(X)]$\<rev(A,B)\>$}
        \subitem The reverse relation on lists. The refinement is a SMT-LIB
        quadratic time predicate function checking whether the arguments are each
        others' reversal.
    \item \texttt{map:[A:list(X),B:list(Y),F:[X,Y]]$\<length(A)=length(B)\>$}
        \subitem The higher-order map relation with a refinement concerning
        lengths.
    \item \texttt{reduceback:[list(T),list(T),[list(T),T,list(T)]]$\<true\>$}
        \subitem The reduceback relation which essentially replaces \texttt{cons}
        in a list with the function parameter. The current refinement system
        cannot assign a more useful refinement.
    \item \texttt{filter:[A:list(T),B:list(T),F:[T]]$\<length(B) \leq length(A)\>$}
        \subitem The filter relation whereby element can only be retained
        from the first list if the \texttt{F} predicate holds for the element.
\end{itemize}

As part of the experiment we add random background predicates. These
predicates drop an element at a fixed index from a list. The refinement for
these predicates encode that lists of lower length are unaffected and lists
of length equal to this index, or higher, have their lengths reduced by one.

The SMT solver's timeout has been set to $30$ milliseconds, which still gives
enough time for the solver to prove some problem instances \texttt{unsat}.
We use the CVC4 SMT solver, for which we specify the logic DTLIA,
with options for finite model finding and inductive reasoning enabled.

For each number of additional background predicates we run ten trials,
for each trial generating positive examples and random predicates.
We average over the trials and calculate the standard deviation of the sample.

\paragraph{Result}

The plots in figure \ref{fig:ref-droplast} depict the average time and number of
proof steps required by the untyped, simply typed and refinement type systems.
The standard deviations are included as bars.

The amount of time necessary for refinement checking is an issue made obvious
by the first graph. The proportion of type spent in Prolog running the Metagol
code is trivial compared to the time spent waiting for the SMT solver. We are
aware that this means that the current implementation does not afford a lot of
practical benefit, instead we should approach the work on refinement types 
as a proof of concept.

For the plot on the right we have that the refinement system is able to cut down on
the search space explored relative to the untyped and simply typed systems,
though with considerable variance\footnote{The sudden increase of variance is
    a sign that this experiment is not particularly well designed. In future
work other long running experiments will be used to evaluate the work.}.
Note that because the types for the most part compose that the difference between
the simply typed and untyped systems is small, again showing that worst case proportional
constant is significant in improving on the untyped system.
Given these graphs we are inclined to 
reject the hypothesis that refinement type checking does not have any advantages,
though we do so with the knowledge that the amount of data generated is actually 
rather limited and a sudden spike in variance. Another provision is that improvement in search space reduction
comes at a very severe cost in regard to execution time.

As a sanity check of the implementations we use the soundness results from this, and the previous, chapter. By soundness we know that when a system prunes part of the search space
they do so only when that part of the search space cannot yield a successful program.
Hence the number of proof steps for the \metast{} should always be at most the number
of proof steps of the \metaai{} system, and the \metart{} system should always
need at most the number of proof steps of the \metast{} system.
We can confirm this for each separate trial that was run, giving some confidence
in the correctness of the implementation.

\begin{figure}[H]
\begin{subfigure}{\textwidth}
\centering
\includegraphics[width=0.99\textwidth]{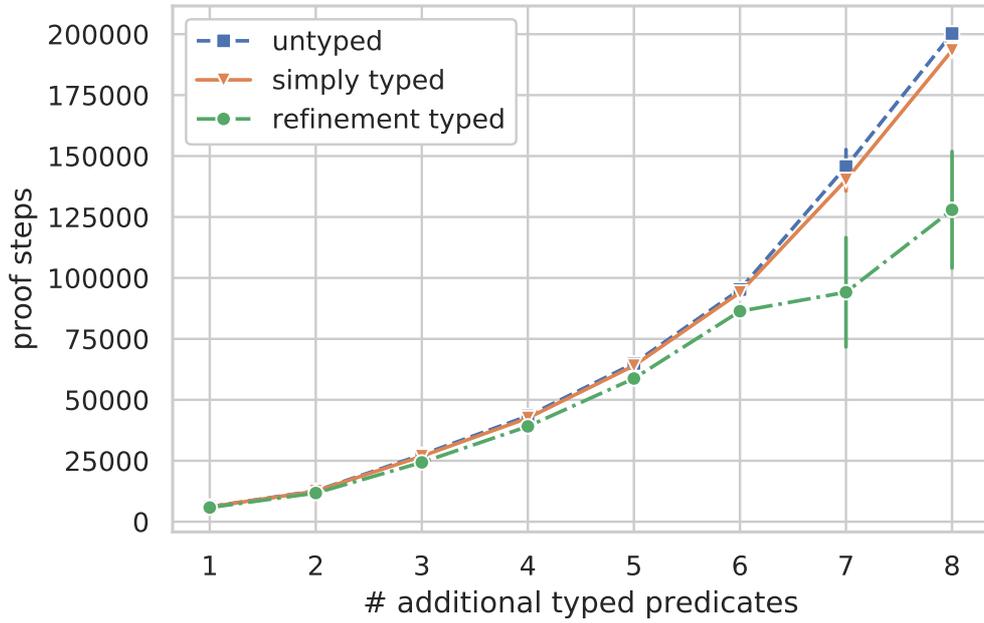}
\end{subfigure}

\hspace{1.5ex}
\begin{subfigure}{\textwidth}
\centering
\includegraphics[width=0.95\textwidth]{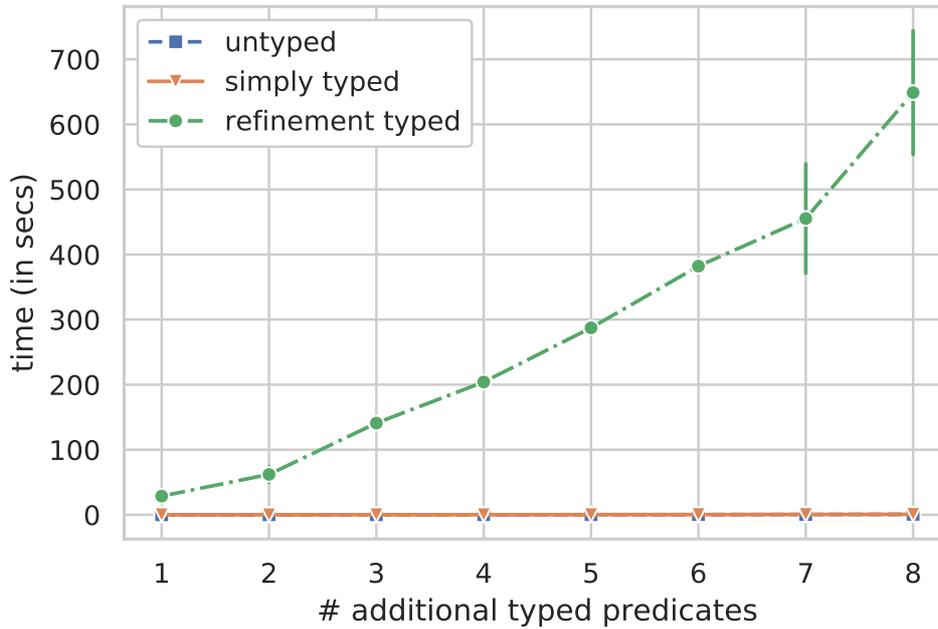}
\end{subfigure}
\caption{Average number of proof steps and time for increasing number of typed background predicates.
Standard error is depicted by bars.}
\label{fig:ref-droplast}
\end{figure}

\stepcounter{expcount}
\subsection{Experiment \arabic{expcount}: Droplasts with Additional Predicates}

In order to distinguish more between the untyped and the simply typed cases
in the previous experiment an (even) longer running experiment was conducted.

\paragraph{Setup}
We take an identical setup to the previous experiment, though now remove
the identity predicate from the base background knowledge and add the 
following refinement typed predicates:

%\begin{center}
%\begin{tabular}{c}
\begin{lstlisting}
dumb0([0],[]):[list(nat),list(X)](*$\<true\>$*).
dumb1(W,0):[list(nat),int](*$\<false\>$*):-findall(K,(between(3,4,K)),W).
dumb2(W,0):[list(int),int](*$\<false\>$*):-findall(K,(between(2,7,K)),W).
\end{lstlisting}
%\end{tabular}
%\end{center}

The first and the second predicate should usually be ruled out due to polymorphic
types and the second and third have refinements that state that they cannot 
be used in a program.

\paragraph{Result}

The plots in figure \ref{fig:ref-droplast-more} depict the average time and number of
proof steps required by the untyped, simply typed and refinement type systems.
The standard deviations are included as bars.
\begin{figure}[h]
\begin{subfigure}{\textwidth}
\centering
\includegraphics[width=0.99\textwidth]{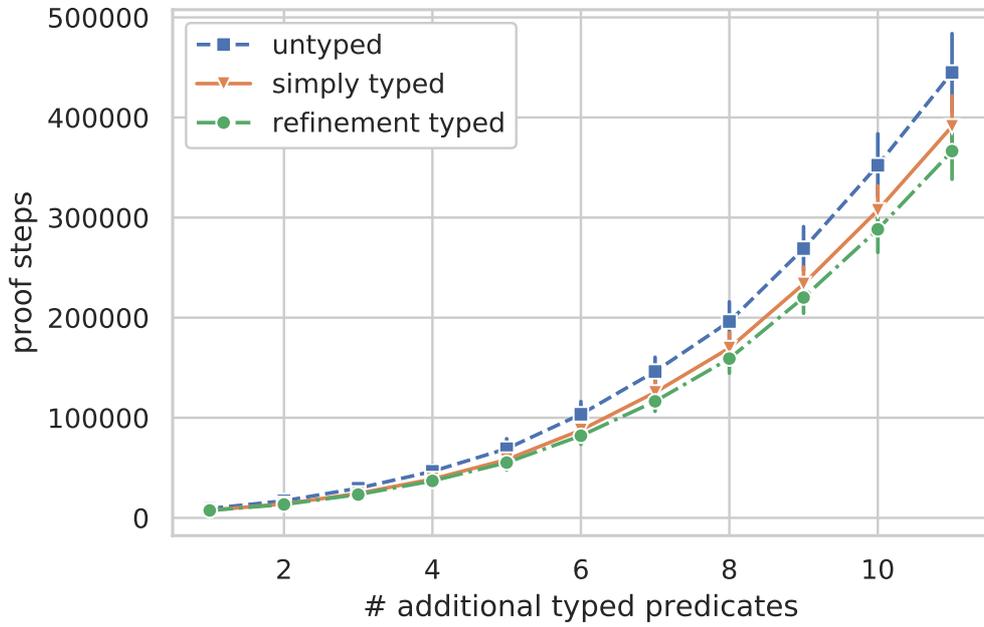}
\end{subfigure}

\hspace{1.5ex}
\begin{subfigure}{\textwidth}
\centering
\includegraphics[width=0.97\textwidth]{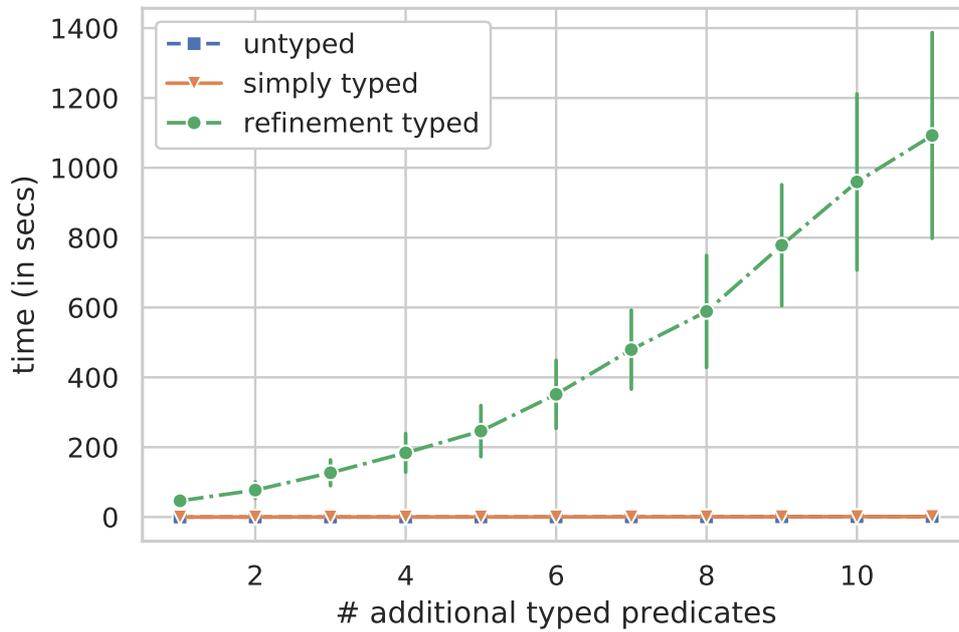}
\end{subfigure}
\caption{Average number of proof steps and time for increasing number of typed background predicates.
Standard error is depicted by bars.}
\label{fig:ref-droplast-more}
\end{figure}

As is unsurprising, the time results show that our refinement checking is
very slow, and has very high variance as well. The plot of the number of
proofs steps to the right shows that while there is evidence for concluding
that the refinement type checking system can improve in regard to the 
untyped system, the simply typed and refinement typed systems are quite close
in the size of the search space traversed. This experiment is
not able to show significant improvements for the refinement type checking.
A better considered experiment is needed, which we have to defer to future
work.

The data does show that the refinement type checking system is able to 
do with strictly fewer proof steps than the (simple) polymorphic typed
system and the untyped system.
We can cautiously use this as evidence
that the refinement typed system is able to improve on the untyped (and simply
typed) system, though we have to keep in mind that this is solely for a
minor reduction in the size of the search space and that it comes at a very
large expense in terms of runtime.

% Scrap the type guided search for now
\iffalse
\chapter{Type Guided Search}

%\todo{explain Metagol depth first search approach in a (couple of) paragraph(s)}

\section{Introducing Limits for First Order Guesses}

%\todo{Explain how the incompleteness in Metagol regarding infinite groudings 
%(e.g. for guessing lists) can be tackled. This solves an outstanding issues
%in Metagol with trying to prove primitives such as `last(X,0)' by instantiating
%X to [0],[\_1,0],[\_2,\_1,0],etc.}

%\todo{Make note of not being able to solve Prolog's incompleteness when it 
%tries to search for (additional) instantiations for for example `droplast(A,B):-reverse(A,C),reverse(D,B),tail(C,D)' when B is any grounded list.}

\section{The Information Available in a Body with Holes}

\section{Best-First Search}

\section{Theoretical Results}

%\todo{Start by explaining Bi-directional Search interpretation}
\fi

\chapter{Conclusion}

Logic program synthesis by Meta-Interpretive Learning (MIL) benefits significantly, 
in terms of search space explored and time needed, from type checking.
The difference in performance between the untyped \metaai{} algorithm 
and the same algorithm with simple type checking, \metast{}, is impressive.
On the other hand, the value of the refinement types checking introduced in 
this document is foremost theoretical: 
there is only limited experimental improvement,
but the work does represent an advancement in terms of bringing refinement
typing to Inductive Logic Programming (ILP).

\subsubsection{Summary and evaluation}
In building on the MIL framework, the basic synthesis algorithm has shown
why it is such a convincing approach to ILP: it is a simple, succinct, and 
highly adaptable extension of the algorithm used for proving atoms in logic 
programming.
The \metaai{} variant of the framework was essential in allowing the system
to express higher-order programs. It is this work on higher-order programs 
that made it possible to consider interesting refinement type properties. 
The major deficiency we identified with the system was its naiveté in not
taking typing into account.

The addition of polymorphic types to the MIL algorithm is a very good fit.
The syntax for  types has been chosen such that it is easy to leverage
the powerful unification of Prolog, which is entirely responsible for type 
checking. In experimental work we have shown that the introduction of simple
type checking is very effective, showing significant improvements in both
shrinking the search space explored and the time required for synthesis.
There is (up to) a cubic reduction in the size of the search space and 
synthesis time, in terms of the number of typed background predicates.
As annotating background knowledge with polymorphic types is only
a small burden for users, this result provides a strong argument
for taking polymorphic type checking into consideration for future ILP
systems.

We presented the theory that introduces refinement types to ILP.
Refinements on types allow for more accurate type checking. The cost of
checking the refinements is however not insignificant. We leverage SMT solving,
which in the current approach incurs such overhead that timewise the refinement
checking has a severe detrimental effect.
The overhead is attributable to the logics considered not being as performant
as hoped, and to the shear number of invocation of the solver.
The experimental work shows that some additional pruning of the search space is 
achieved, though more work is needed to make refinement type checking in ILP
worthwhile.

\subsubsection{Future work}

\paragraph{Theory}
The experimental work on polymorphic types (section \ref{sec:simple-expr}) showed
(and partially explained) a discrepancy between the theory of the size of
hypothesis spaces and the search space over programs considered by MIL algorithms.
A better characterization of the search spaces considered by MIL would be useful
in predicting the impact of algorithmic changes. 

The work regarding
polymorphic type checking itself can be made more convincing by presenting
formal proofs of soundness and completeness. In the future a formal type
system should be introduced for this purpose.

\paragraph{Justifying refinements}
While the polymorphic type checking approach is quite satisfactory, the refinement 
types work can be improved upon in a number of areas.
Following on from the experimental work, better experiments are needed
to justify refinement type checking as a worthwhile approach to further
pruning the search space.

In addition, as the main aim of refinement
type checking is in improving performance, the most direct way of addressing
this issue is by making use of more performant theories for the SMT solvers.
The work is mainly in identifying logics that are expressive enough to state useful
properties, while requiring much less time to prove (un)satisfiability.
Analysis of the number of invocations of the SMT solvers in the experimental data 
of section \ref{sec:ref-expr} could be used to express the performance needed 
of SMT solvers for the current approach, which involves many SMT invocations,
to be a sensible way forward. 

\paragraph{Reduction to SMT}
The current approach asks the SMT solver to solve an entirely new problem
every time it is invoked, but when a previous refinement check was
satisfiable it is usually the case that only additional assertions need
to be added. This structural property could be leveraged to make solvers
more efficient by allowing them to reuse work performed for the previous checks.
A more ambitious approach (briefly considered for this project) is to completely
encode the ILP synthesis problem as a set of constraints for a SMT solver.
The prospect of integrating the checking of refinements into such a single
SMT problem is especially enticing.

\paragraph{Directing the search}
In section \ref{sec:typed-direct-search} we already explored the possibility of
types being able to guide the traversal of the search space, which given a 
well-considered heuristic would be able to further prune the search space.
Such further pruning might also have implications for the viability of the
current refinement type checking approach as this could significantly
reduce the number of invocations of the SMT solver.

\paragraph{Functional metarules}
The MIL approach to synthesis is especially powerful in that it is able to
invent program clauses. The main reason for this capability are the metarules.
It appears that introducing rules for structural inventions might
be applicable outside of ILP. Introduction of metarules to the setting
of functional programs would a major new avenue to explore.
Existing type system-based approaches could be extended, thereby
introducing the (almost) unique feature of invention of helper clauses/functions
to the field of functional program synthesis.

\bibliographystyle{plainnat}

\bibliography{../library,./bibliography}

\begin{thebibliography}{51}
\providecommand{\natexlab}[1]{#1}
\providecommand{\url}[1]{\texttt{#1}}
\expandafter\ifx\csname urlstyle\endcsname\relax
  \providecommand{\doi}[1]{doi: #1}\else
  \providecommand{\doi}{doi: \begingroup \urlstyle{rm}\Url}\fi

\bibitem[B et~al.(2017)B, Koutris, Naik, and Smith]{B2017}
Aws~Albarghouthi B, Paraschos Koutris, Mayur Naik, and Calvin Smith.
\newblock {Constraint-Based Synthesis of Datalog Programs}.
\newblock 10416:\penalty0 689--706, 2017.
\newblock \doi{10.1007/978-3-319-66158-2}.
\newblock URL \url{http://link.springer.com/10.1007/978-3-319-66158-2}.

\bibitem[Bar-David and Taubenfeld(2003)]{bar2003automatic}
Yoah Bar-David and Gadi Taubenfeld.
\newblock Automatic discovery of mutual exclusion algorithms.
\newblock In \emph{International Symposium on Distributed Computing}, pages
  136--150. Springer, 2003.

\bibitem[Barrett et~al.(2010)Barrett, Stump, Tinelli, et~al.]{barrett2010smt}
Clark Barrett, Aaron Stump, Cesare Tinelli, et~al.
\newblock The smt-lib standard: Version 2.0.
\newblock In \emph{Proceedings of the 8th International Workshop on
  Satisfiability Modulo Theories (Edinburgh, England)}, volume~13, page~14,
  2010.

\bibitem[Barrett et~al.(2011)Barrett, Conway, Deters, Hadarean, Jovanovi{\'c},
  King, Reynolds, and Tinelli]{barrett2011cvc4}
Clark Barrett, Christopher~L Conway, Morgan Deters, Liana Hadarean, Dejan
  Jovanovi{\'c}, Tim King, Andrew Reynolds, and Cesare Tinelli.
\newblock Cvc4.
\newblock In \emph{International Conference on Computer Aided Verification},
  pages 171--177. Springer, 2011.

\bibitem[Bibel(2013)]{bibel2013automated}
Wolfgang Bibel.
\newblock \emph{Automated theorem proving}.
\newblock Springer Science \& Business Media, 2013.

\bibitem[Bojarski et~al.(2016)Bojarski, Del~Testa, Dworakowski, Firner, Flepp,
  Goyal, Jackel, Monfort, Muller, Zhang, et~al.]{bojarski2016end}
Mariusz Bojarski, Davide Del~Testa, Daniel Dworakowski, Bernhard Firner, Beat
  Flepp, Prasoon Goyal, Lawrence~D Jackel, Mathew Monfort, Urs Muller, Jiakai
  Zhang, et~al.
\newblock End to end learning for self-driving cars.
\newblock \emph{arXiv preprint arXiv:1604.07316}, 2016.

\bibitem[Coq(2018)]{the_coq_development_team_2018_1219885}
Coq.
\newblock The coq proof assistant, version 8.8.0, apr 2018.
\newblock URL \url{https://doi.org/10.5281/zenodo.1219885}.

\bibitem[Cropper(2017)]{Cropper2017}
Andrew Cropper.
\newblock \emph{{Efficiently learning efficient programs}}.
\newblock PhD thesis, Imperial College London, 2017.

\bibitem[Cropper and Muggleton(2016{\natexlab{a}})]{Cropper2016}
Andrew Cropper and Stephen~H. Muggleton.
\newblock {Learning higher-order logic programs through abstraction and
  invention}.
\newblock \emph{IJCAI Int. Jt. Conf. Artif. Intell.}, 2016-Janua:\penalty0
  1418--1424, 2016{\natexlab{a}}.
\newblock ISSN 10450823.

\bibitem[Cropper and Muggleton(2016{\natexlab{b}})]{metagol}
Andrew Cropper and Stephen~H. Muggleton.
\newblock Metagol system.
\newblock https://github.com/metagol/metagol, 2016{\natexlab{b}}.
\newblock URL \url{https://github.com/metagol/metagol}.

\bibitem[Cropper and Muggleton(2018)]{Cropper2018}
Andrew Cropper and Stephen~H. Muggleton.
\newblock Learning efficient logic programs.
\newblock \emph{Machine Learning}, Apr 2018.
\newblock ISSN 1573-0565.
\newblock \doi{10.1007/s10994-018-5712-6}.
\newblock URL \url{https://doi.org/10.1007/s10994-018-5712-6}.

\bibitem[Cropper et~al.(2016)Cropper, Tamaddoni-Nezhad, and
  Muggleton]{cropper2016data}
Andrew Cropper, Alireza Tamaddoni-Nezhad, and Stephen~H. Muggleton.
\newblock Meta-interpretive learning of data transformation programs.
\newblock In Katsumi Inoue, Hayato Ohwada, and Akihiro Yamamoto, editors,
  \emph{Inductive Logic Programming}, pages 46--59, Cham, 2016. Springer
  International Publishing.

\bibitem[De~Moura and Bj{\o}rner(2008)]{de2008z3}
Leonardo De~Moura and Nikolaj Bj{\o}rner.
\newblock Z3: An efficient smt solver.
\newblock In \emph{International conference on Tools and Algorithms for the
  Construction and Analysis of Systems}, pages 337--340. Springer, 2008.

\bibitem[Deng et~al.(2009)Deng, Dong, Socher, Li, Li, and
  Fei-Fei]{deng2009imagenet}
Jia Deng, Wei Dong, Richard Socher, Li-Jia Li, Kai Li, and Li~Fei-Fei.
\newblock Imagenet: A large-scale hierarchical image database.
\newblock In \emph{Computer Vision and Pattern Recognition, 2009. CVPR 2009.
  IEEE Conference on}, pages 248--255. Ieee, 2009.

\bibitem[Farquhar et~al.(2015)Farquhar, Grov, Cropper, Muggleton, and
  Bundy]{Farquhar2015}
Colin Farquhar, Gudmund Grov, Andrew Cropper, Stephen Muggleton, and Alan
  Bundy.
\newblock {Typed meta-interpretive learning for proof strategies}.
\newblock \emph{CEUR Workshop Proc.}, 1636:\penalty0 17--32, 2015.
\newblock ISSN 16130073.

\bibitem[Frankle et~al.(2016)Frankle, Osera, Walker, and
  Zdancewic]{frankle2016example}
Jonathan Frankle, Peter-Michael Osera, David Walker, and Steve Zdancewic.
\newblock Example-directed synthesis: a type-theoretic interpretation.
\newblock \emph{ACM SIGPLAN Notices}, 51\penalty0 (1):\penalty0 802--815, 2016.

\bibitem[Freeman and Pfenning(1991)]{freeman1991refinement}
Tim Freeman and Frank Pfenning.
\newblock \emph{Refinement types for ML}, volume~26.
\newblock ACM, 1991.

\bibitem[Girard(1971)]{girard1971extension}
J-Y Girard.
\newblock Une extension de l'interpretation de godel a l'analyse et son
  application a l'elimination des coupures dans l'analyse et la theorie des
  types.
\newblock In \emph{Proc. 2nd Scandinavian Logic Symp.}, pages 63--92.
  North-Holland, 1971.

\bibitem[Gulwani(2010)]{Gulwani2010}
Sumit Gulwani.
\newblock {Dimensions in program synthesis}.
\newblock \emph{Proc. 12th Int. ACM SIGPLAN Symp. Princ. Pract. Declar.
  Program. - PPDP '10}, pages 13--24, 2010.
\newblock \doi{10.1145/1836089.1836091}.
\newblock URL \url{http://portal.acm.org/citation.cfm?doid=1836089.1836091}.

\bibitem[Gulwani et~al.(2015)Gulwani, Hern{\'a}ndez-Orallo, Kitzelmann,
  Muggleton, Schmid, and Zorn]{gulwani2015inductive}
Sumit Gulwani, Jos{\'e} Hern{\'a}ndez-Orallo, Emanuel Kitzelmann, Stephen~H
  Muggleton, Ute Schmid, and Benjamin Zorn.
\newblock Inductive programming meets the real world.
\newblock \emph{Communications of the ACM}, 58\penalty0 (11):\penalty0 90--99,
  2015.

\bibitem[Gvero et~al.(2013)Gvero, Kuncak, Kuraj, and Piskac]{Gvero2013}
Tihomir Gvero, Viktor Kuncak, Ivan Kuraj, and Ruzica Piskac.
\newblock {Complete completion using types and weights}.
\newblock \emph{Pldi}, 48\penalty0 (6):\penalty0 27--38, 2013.
\newblock ISSN 15232867.
\newblock \doi{10.1145/2491956.2462192}.
\newblock URL
  \url{http://dl.acm.org/citation.cfm?doid=2499370.2462192{\%}5Cnpapers3://publication/doi/10.1145/2499370.2462192}.

\bibitem[Jha et~al.(2010)Jha, Gulwani, Seshia, and
  Tiwari]{Jha:2010:OCP:1806799.1806833}
Susmit Jha, Sumit Gulwani, Sanjit~A. Seshia, and Ashish Tiwari.
\newblock Oracle-guided component-based program synthesis.
\newblock In \emph{Proceedings of the 32Nd ACM/IEEE International Conference on
  Software Engineering - Volume 1}, ICSE '10, pages 215--224, New York, NY,
  USA, 2010. ACM.
\newblock ISBN 978-1-60558-719-6.
\newblock \doi{10.1145/1806799.1806833}.
\newblock URL \url{http://doi.acm.org/10.1145/1806799.1806833}.

\bibitem[Kanellakis and Mitchell(1989)]{kanellakis1989polymorphic}
Paris~C Kanellakis and John~C Mitchell.
\newblock Polymorphic unification and ml typing.
\newblock In \emph{Proceedings of the 16th ACM SIGPLAN-SIGACT symposium on
  Principles of programming languages}, pages 105--115. ACM, 1989.

\bibitem[Katayama(2005)]{Katayama2005}
Susumu Katayama.
\newblock {Systematic search for lambda expressions}.
\newblock \emph{Trends Funct. Program.}, 81\penalty0 (985):\penalty0 195--205,
  2005.
\newblock URL
  \url{http://books.google.com/books?hl=en{\&}lr={\&}id=p0yV1sHLubcC{\&}oi=fnd{\&}pg=PA111{\&}dq=Systematic+search+for+lambda+expressions{\&}ots=x57nz-UKrs{\&}sig=I3{\_}ntlWzKChXekOKKsvtQH2FEvg}.

\bibitem[Kitzelmann(2007)]{kitzelmann2007data}
Emanuel Kitzelmann.
\newblock Data-driven induction of recursive functions from
  input/output-examples.
\newblock In \emph{Proceedings of the ECML/PKDD 2007 Workshop on Approaches and
  Applications of Inductive Programming (AAIP’07).}, 2007.

\bibitem[Kitzelmann(2010)]{kitzelmann}
Emanuel Kitzelmann.
\newblock Inductive programming: A survey of program synthesis techniques.
\newblock In Ute Schmid, Emanuel Kitzelmann, and Rinus Plasmeijer, editors,
  \emph{Approaches and Applications of Inductive Programming}, pages 50--73,
  Berlin, Heidelberg, 2010. Springer Berlin Heidelberg.
\newblock ISBN 978-3-642-11931-6.

\bibitem[Lau et~al.(2000)Lau, Domingos, and Weld]{Lau2000}
Tessa~A. Lau, Pedro Domingos, and Daniel~S. Weld.
\newblock {Version Space Algebra and its Application to Programming by
  Demonstration}.
\newblock \emph{ICML '00 Proc. Seventeenth Int. Conf. Mach. Learn.}, pages
  527--534, 2000.
\newblock URL \url{http://dl.acm.org/citation.cfm?id=657973}.

\bibitem[Lin et~al.(2014)Lin, Dechter, Ellis, Tenenbaum, and
  Muggleton]{lin2014bias}
Dianhuan Lin, Eyal Dechter, Kevin Ellis, Joshua~B Tenenbaum, and Stephen~H
  Muggleton.
\newblock Bias reformulation for one-shot function induction.
\newblock 2014.

\bibitem[Massalin(1987)]{Massalin:1987:SLS:36177.36194}
Henry Massalin.
\newblock Superoptimizer: A look at the smallest program.
\newblock \emph{SIGARCH Comput. Archit. News}, 15\penalty0 (5):\penalty0
  122--126, October 1987.
\newblock ISSN 0163-5964.
\newblock \doi{10.1145/36177.36194}.
\newblock URL \url{http://doi.acm.org/10.1145/36177.36194}.

\bibitem[Michie(1988)]{michie}
D.~Michie.
\newblock Machine learning in the next five years.
\newblock In \emph{Third European working session on learning, In Proceedings
  of the}, pages 107--122, 1988.

\bibitem[Muggleton(1991)]{Muggleton1991}
Stephen~H. Muggleton.
\newblock {Inductive Logic Programming}.
\newblock \emph{New Gener. Comput.}, 8\penalty0 (4):\penalty0 295--318, 1991.
\newblock ISSN 0288-3635.
\newblock \doi{10.1007/BFb0027303}.
\newblock URL \url{http://link.springer.com/10.1007/BFb0027303}.

\bibitem[Muggleton et~al.(2014{\natexlab{a}})Muggleton, Lin, Pahlavi, and
  Tamaddoni-Nezhad]{Muggleton2014}
Stephen~H. Muggleton, Dianhuan Lin, Niels Pahlavi, and Alireza
  Tamaddoni-Nezhad.
\newblock {Meta-interpretive learning: Application to grammatical inference}.
\newblock \emph{Mach. Learn.}, 94\penalty0 (1):\penalty0 25--49,
  2014{\natexlab{a}}.
\newblock ISSN 08856125.
\newblock \doi{10.1007/s10994-013-5358-3}.

\bibitem[Muggleton et~al.(2014{\natexlab{b}})Muggleton, Lin, Pahlavi, and
  Tamaddoni-Nezhad]{muggleton2014meta}
Stephen~H Muggleton, Dianhuan Lin, Niels Pahlavi, and Alireza Tamaddoni-Nezhad.
\newblock Meta-interpretive learning: application to grammatical inference.
\newblock \emph{Machine learning}, 94\penalty0 (1):\penalty0 25--49,
  2014{\natexlab{b}}.

\bibitem[Muggleton et~al.(2015)Muggleton, Lin, and
  Tamaddoni-Nezhad]{Muggleton2015}
Stephen~H. Muggleton, Dianhuan Lin, and Alireza Tamaddoni-Nezhad.
\newblock {Meta-interpretive learning of higher-order dyadic datalog: predicate
  invention revisited}.
\newblock \emph{Mach. Learn.}, 100\penalty0 (1):\penalty0 49--73, 2015.
\newblock ISSN 15730565.
\newblock \doi{10.1007/s10994-014-5471-y}.

\bibitem[Muggleton et~al.(2018)Muggleton, Schmid, Zeller, Tamaddoni-Nezhad, and
  Besold]{Muggleton2018}
Stephen~H. Muggleton, Ute Schmid, Christina Zeller, Alireza Tamaddoni-Nezhad,
  and Tarek Besold.
\newblock Ultra-strong machine learning: comprehensibility of programs learned
  with ilp.
\newblock \emph{Machine Learning}, 107\penalty0 (7):\penalty0 1119--1140, Jul
  2018.
\newblock ISSN 1573-0565.
\newblock \doi{10.1007/s10994-018-5707-3}.
\newblock URL \url{https://doi.org/10.1007/s10994-018-5707-3}.

\bibitem[N.~Bjørner and Veanes.(2012)]{SMTseq}
R.~Michel N.~Bjørner, V.~Ganesh and M.~Veanes.
\newblock An smt-lib format for sequences and regular expressions.
\newblock In \emph{In SMT workshop 2012}, 2012.

\bibitem[Nienhuys-Cheng and {De Wolf}(1997)]{nienhuys1997foundations}
Shan-Hwei Nienhuys-Cheng and Ronald {De Wolf}.
\newblock \emph{{Foundations of inductive logic programming}}, volume 1228.
\newblock Springer Science {\&} Business Media, 1997.

\bibitem[Nipkow et~al.(2002)Nipkow, Wenzel, and
  Paulson]{Nipkow:2002:IPA:1791547}
Tobias Nipkow, Markus Wenzel, and Lawrence~C. Paulson.
\newblock \emph{Isabelle/HOL: A Proof Assistant for Higher-order Logic}.
\newblock Springer-Verlag, Berlin, Heidelberg, 2002.
\newblock ISBN 3-540-43376-7.

\bibitem[Osera and Zdancewic(2015)]{osera2015type}
Peter-Michael Osera and Steve Zdancewic.
\newblock Type-and-example-directed program synthesis.
\newblock In \emph{ACM SIGPLAN Notices}, volume~50, pages 619--630. ACM, 2015.

\bibitem[Polikarpova et~al.(2016)Polikarpova, Kuraj, and
  Solar-Lezama]{polikarpova2016program}
Nadia Polikarpova, Ivan Kuraj, and Armando Solar-Lezama.
\newblock Program synthesis from polymorphic refinement types.
\newblock \emph{ACM SIGPLAN Notices}, 51\penalty0 (6):\penalty0 522--538, 2016.

\bibitem[Raedt(2010)]{Raedt2010}
Luc~De Raedt.
\newblock \emph{Inductive Logic Programming}, pages 529--537.
\newblock Springer US, Boston, MA, 2010.
\newblock ISBN 978-0-387-30164-8.
\newblock \doi{10.1007/978-0-387-30164-8_396}.
\newblock URL \url{https://doi.org/10.1007/978-0-387-30164-8_396}.

\bibitem[Raychev et~al.(2014)Raychev, Vechev, and
  Yahav]{Raychev:2014:CCS:2666356.2594321}
Veselin Raychev, Martin Vechev, and Eran Yahav.
\newblock Code completion with statistical language models.
\newblock \emph{SIGPLAN Not.}, 49\penalty0 (6):\penalty0 419--428, June 2014.
\newblock ISSN 0362-1340.
\newblock \doi{10.1145/2666356.2594321}.
\newblock URL \url{http://doi.acm.org/10.1145/2666356.2594321}.

\bibitem[Reynolds and Blanchette(2015)]{reynolds2015decision}
Andrew Reynolds and Jasmin~Christian Blanchette.
\newblock A decision procedure for (co) datatypes in smt solvers.
\newblock In \emph{International Conference on Automated Deduction}, pages
  197--213. Springer, 2015.

\bibitem[Reynolds et~al.(2016)Reynolds, Blanchette, Cruanes, and
  Tinelli]{reynolds2016model}
Andrew Reynolds, Jasmin~Christian Blanchette, Simon Cruanes, and Cesare
  Tinelli.
\newblock Model finding for recursive functions in smt.
\newblock In \emph{International Joint Conference on Automated Reasoning},
  pages 133--151. Springer, 2016.

\bibitem[Singh and Gulwani(2012{\natexlab{a}})]{Singh2012}
Rishabh Singh and Sumit Gulwani.
\newblock {Learning semantic string transformations from examples}.
\newblock \emph{Proc. VLDB Endow.}, 5\penalty0 (8):\penalty0 740--751,
  2012{\natexlab{a}}.
\newblock ISSN 21508097.
\newblock \doi{10.14778/2212351.2212356}.
\newblock URL \url{http://dl.acm.org/citation.cfm?doid=2212351.2212356}.

\bibitem[Singh and Gulwani(2012{\natexlab{b}})]{Singh2012a}
Rishabh Singh and Sumit Gulwani.
\newblock {LNCS 7358 - Synthesizing Number Transformations from Input-Output
  Examples}.
\newblock \emph{Cav}, pages 634--651, 2012{\natexlab{b}}.
\newblock URL
  \url{https://www.microsoft.com/en-us/research/wp-content/uploads/2016/06/cav12-1-1.pdf}.

\bibitem[Solar-Lezama(2009)]{Solar-Lezama2009}
Armando Solar-Lezama.
\newblock {The sketching approach to program synthesis}.
\newblock \emph{Lect. Notes Comput. Sci. (including Subser. Lect. Notes Artif.
  Intell. Lect. Notes Bioinformatics)}, 5904 LNCS:\penalty0 4--13, 2009.
\newblock ISSN 03029743.
\newblock \doi{10.1007/978-3-642-10672-9_3}.

\bibitem[S{\o}rensen and Urzyczyn(2006)]{sorensen2006lectures}
Morten~Heine S{\o}rensen and Pawel Urzyczyn.
\newblock \emph{Lectures on the Curry-Howard isomorphism}, volume 149.
\newblock Elsevier, 2006.

\bibitem[Srivastava et~al.(2013)Srivastava, Gulwani, and
  Foster]{Srivastava2013}
Saurabh Srivastava, Sumit Gulwani, and Jeffrey~S. Foster.
\newblock {Template-based program verification and program synthesis}.
\newblock \emph{Int. J. Softw. Tools Technol. Transf.}, 15\penalty0
  (5-6):\penalty0 497--518, 2013.
\newblock ISSN 14332779.
\newblock \doi{10.1007/s10009-012-0223-4}.

\bibitem[Wang et~al.(2017)Wang, Dillig, and Singh]{Wang2017}
Xinyu Wang, Isil Dillig, and Rishabh Singh.
\newblock {Program Synthesis using Abstraction Refinement}.
\newblock 1\penalty0 (January), 2017.
\newblock \doi{10.1145/3158151}.
\newblock URL \url{http://arxiv.org/abs/1710.07740}.

\bibitem[Warren(2002)]{Warren:2002:HD:515297}
Henry~S. Warren.
\newblock \emph{Hacker's Delight}.
\newblock Addison-Wesley Longman Publishing Co., Inc., Boston, MA, USA, 2002.
\newblock ISBN 0201914654.

\end{thebibliography}

%\chapter*{Appendix}

\end{document}